\documentclass[11pt]{article}

\usepackage[preprint]{acl}

\usepackage{times}
\usepackage{latexsym}

\usepackage[T1]{fontenc}

\usepackage[utf8]{inputenc}

\usepackage{microtype}

\usepackage{inconsolata}

\usepackage{graphicx}

\usepackage{xcolor}
\usepackage{tikz}
\usetikzlibrary{positioning,arrows.meta,fit,calc,shadows,backgrounds}
\usepackage{array}

\usepackage{amsmath}
\usepackage{amssymb}
\usepackage{mathtools}
\usepackage{amsthm}
\usepackage{bbm}
\usepackage{subcaption}
\usepackage{booktabs}
\usepackage{multirow}
\usepackage[capitalize,noabbrev]{cleveref}

\theoremstyle{plain}
\newtheorem{theorem}{Theorem}[section]
\newtheorem{proposition}[theorem]{Proposition}

\newtheorem{corollary}[theorem]{Corollary}
\theoremstyle{definition}
\newtheorem{definition}[theorem]{Definition}
\newtheorem{assumption}[theorem]{Assumption}
\theoremstyle{remark}
\newtheorem{remark}[theorem]{Remark}
\title{Every Token Counts: Exact Likert-Scale Distributions for Measuring LLM Attitudes and Biases}

\author{Davood Wadi \\ McGill University \And  
        Mohsen Ghodrat \\ University Canada West \And
        Matthew Philp \\ Marketing Management \\ Toronto Metropolitan University}

\begin{document}
\maketitle

\begin{abstract}
As Large Language Models (LLMs) are increasingly deployed as autonomous agents, accurately evaluating their latent values and biases is critical. The NLP community typically evaluates models using large, unstructured benchmarks. While effective for general capabilities, these datasets fundamentally conflate causal mechanisms: even when an aggregate bias is detected, unstructured evaluations cannot disentangle whether it stems from baseline traits, contextual confounders, or complex interactions. 
To address this, we introduce an analytically exact framework for the controlled behavioral evaluation of LLMs. We bridge human psychometrics with LLM mechanics by resolving gaps in design, measurement, and analysis. First, we replace unstructured prompting with fully crossed factorial experiments to systematically isolate causal main and interaction effects. Second, we eliminate Monte Carlo text sampling noise by operating directly on exact, token-level Probability Mass Functions (PMFs). Third, we derive a multivariate ordinal consensus metric and a distributional ANOVA to process these PMFs analytically. 
We validate our framework with a case study on consumer ethnocentrism across five LLMs, demonstrating how our approach isolates systemic country-of-origin biases that aggregate benchmarks otherwise obscure.
\end{abstract}

\section{Introduction}

As Large Language Models (LLMs) are increasingly deployed as autonomous agents \citep{wang2024survey, xi2025rise,wadi2026shopping,wadi2026careful}, evaluating their latent biases and subjective behaviors is critical. A pressing question in this new era is how to accurately quantify latent behavioral constructs, such as personality traits, moral values, or sociopolitical biases, to ensure AI fairness and safe global deployment.

Historically, the NLP community has evaluated models using large, unstructured benchmark datasets \citep{hendrycks2020measuring, bommasani2023holistic}.\footnote{We use unstructured to denote evaluation sets that aggregate large numbers of heterogeneous prompts without a controlled factorial structure, such that the covariates of interest (e.g., demographic, context, task) are not orthogonally crossed and are therefore statistically confounded.}
While this scale-first paradigm is highly effective for measuring general capabilities (e.g., trivia, coding), it is fundamentally ill-equipped for precise behavioral evaluation \citep{pmlr-v202-santurkar23a, durmus2024towards}. 
Even when unstructured benchmarks successfully detect an aggregate bias in a model, they inherently lack the analytical power to disentangle \emph{which} underlying factors actually caused that behavior, obscuring whether a bias is a hardwired training defect or a context-dependent adaptation \citep{wadi2026shopping}.

Consider an intuitive example: evaluating an LLM for prosocial behavior under financial incentives using a large, unstructured dataset of consumption scenarios. 
A benchmark might successfully detect a statistically significant tendency for altruistic behavior \citep{wadi2026interplay}. However, the unstructured dataset cannot disentangle \emph{why} this tendency exists. 
Because the prompt contexts vary across the dataset, altruism might frequently be paired with reciprocal gains \citep{wadi2026careful}. 
Is the model fundamentally altruistic (a \emph{main effect} of altruism)? Is the model simply generating more prosocial text because of financial incentives, (a \emph{main effect} of financial incentives)? 
Or is the model indifferent towards prosocial motives, but severely shows prosocial behavior only when financial incentives are provided (an \emph{interaction effect})? 
A benchmark dataset conflates these distinct causal mechanisms.

To disentangle what drives a model's behavior, we propose a shift from unstructured dataset scaling to controlled experimentation. 
We formulate LLM evaluation as a fully crossed factorial design, where every distinct variable (e.g., demographic, context, LLM persona) is treated as an independent factor. 
We systematically evaluate every possible combination using validated psychometric scales \citep{cronbach1950further, batterton2017likert}. 
This experimental control grants us the analytical power to mathematically isolate causal main effects and higher-order interactions. This fills an important gap in behavioral attribution that aggregate benchmarks lack. 

In standard NLP evaluation, scaling up a dataset with thousands of generated prompts is an excellent way to test general model capabilities. However, if we try to apply this "scale-first" approach to a psychological survey by adding thousands of synthetic, unvalidated questions, we risk altering the core trait we are trying to measure, a problem known as reducing construct validity \citep{flake2020measurement}.

To keep our measurements reliable, we rely on carefully calibrated questions used by behavioral scientists (e.g., Likert scales). Moreover, translating behavioral methods from humans to LLMs has several major methodological gaps, which we formalize as experimental design, measurement, and analysis. 
We formulate a statistical framework to execute controlled behavioral experiments on LLMs, while addressing these gaps.

First, to address the design gap, our framework formulates LLM evaluation as a \emph{fully crossed factorial experiment}. By treating variables (e.g., the specific LLM being tested, the target demographic in the prompt) as distinct experimental factors, we systematically separate a model's baseline behavior from causal main effects and higher-order interactions. This experimental control grants us the analytical power to mathematically isolate biases free of confounding contexts.

Second, to address the measurement gap, we bypass the issues of generative text sampling. Human behavioral science relies on sampling because the underlying probability distribution of a human mind is inaccessible. NLP has adopted this paradigm via Monte Carlo text sampling to simulate a "population" of responses \citep{wadi2025a}, often injecting additional noise via the temperature parameter. However, an LLM possesses a constant, fully observable internal state for a given prompt. Generating multiple text responses is effectively taking repeated, noisy draws from a static mathematical distribution. Our framework, instead, operates directly on exact, token-level Probability Mass Functions (PMFs) to eliminate sampling error.

Third, to address the analytical gap, we formulate the statistical mechanics required to process these exact PMFs. When NLP researchers evaluate token probabilities, they typically rely on metrics like Shannon entropy over valid tokens \citep{farquhar2024detecting}. However, entropy fails on Likert scales because it ignores ordinality. For example, entropy cannot distinguish between an LLM heavily polarized between "Strongly Agree" (5) and "Strongly Disagree" (1) versus one smoothly concentrated around "Neutral" (3). We bridge this gap by introducing a \emph{multivariate consensus} metric that respects ordinal distance. We also formulate \emph{distributional ANOVA} to analytically compute the exact probability distributions of the isolated causal effects.

To empirically validate our general-purpose framework, we present a case study on consumer ethnocentrism, the sociological phenomenon where an entity systematically favors its country of origin and rejects foreign entities \citep{shankarmahesh2006consumer}. Applying established ethnocentrism scales to five modern LLMs, our framework shows how models' countries of origin inherently interact with target countries, thereby allowing us to test for the magnitude of ethnocentrism for each LLM. We mathematically show that controlled experiments can isolate geographic biases without relying on large, unstructured datasets.

We contribute a unified, deterministic framework for the causal behavioral
evaluation of LLMs on ordinal psychometric instruments.
The framework casts evaluation as a fully crossed factorial experiment operating
directly on exact token-level Probability Mass Functions (PMFs), and its
methodological core is a distributional ANOVA that isolates causal effects
as full probability distributions. 
The pipeline:
\begin{itemize}
    \item decomposes the composite construct distribution into baseline, main-effect, and interaction-effect distributions via a distributional
    Hoeffding decomposition with Optimal-Transport-based paired contrasts and a formal ANOVA-consistency guarantee (Theorem~\ref{thm:anova-consistency-main});
    \item propagates item-level PMFs into an exact composite-score distribution using discrete convolution; and
    \item addresses a failure mode of Shannon entropy on Likert scales by quantifying ordinal response certainty using a multivariate extension of the Consensus metric \citep{tastle2007consensus}.
\end{itemize}

\section{Related work}
\label{related-work}

The NLP community has adopted frameworks from the behavioral sciences to study LLMs' latent traits \citep{ward-etal-2024-evaluating, perez-etal-2023-discovering,wadi2026does}. 
Recent works have utilized psychometric tests to measure a wide array of constructs \citep{jiang-etal-2024-personallm, li-etal-2024-evaluating-psychological, helwe-etal-2025-navigating}. 
However, recent studies demonstrate that eliciting psychometric responses via text generation is highly unstable. 
Simple prompt perturbations can significantly downgrade the consistency of LLM responses on persona axes \citep{shu-etal-2024-dont}. 
Moreover, traditional psychometric testing on LLMs suffers from low ecological validity when relying on standard generative outputs \citep{jung-etal-2026-psychometric}. 
Part of this instability stems from generative sampling, which is a methodological artifact borrowed from human surveying,
and obscures their actual underlying distribution \citep{deas-mckeown-2025-artificial}. 

\paragraph{Uncertainty and probability in LLMs}
Recognizing the limitations of text-based sampling, a growing body of literature has turned to analyzing exact token probabilities to better understand LLMs' internal states. 
\citet{kuribayashi-etal-2024-psychometric} demonstrated that utilizing exact next-word probabilities is far superior to instruction-tuned prompting when simulating human cognitive behaviors. 
Similarly, recent work cautions that prompting LLMs for open-ended responses yields poorly calibrated text \citep{liu2025badworktimecrosscultural}, 
whereas their exact internal token probabilities remain highly predictive 
\citep{bentegeac2025token, kadavath2022language}.
To quantify the variance within these predictive probability distributions, 
the NLP community relies almost exclusively on Shannon entropy or its semantic derivatives \citep{farquhar2024detecting}. 
Entropy is highly effective for tasks where the output space consists of distinct, categorical facts (e.g., detecting hallucinations across mutually exclusive answers). 
However, standard metrics like entropy fail when applied to behavioral evaluation using psychometric rating scales. 
This is because entropy treats all tokens as unordered, categorical variables, ignoring the ordinality of a Likert scale. 
Because standard uncertainty metrics are indifferent to this ordinality, we lack the mathematical foundation to utilize exact token probability mass functions (PMFs) for behavioral evaluation. 
As demonstrated in \ref{sec:methods}, our framework resolved this impasse by formulating a measure of consensus specifically designed for ordinal variables.

\paragraph{Geopolitical bias and ethnocentrism in LLMs}

As LLMs are deployed globally, there is growing concern regarding their cultural neutrality. 
A robust body of literature documents that LLMs exhibit systemic 
geopolitical and geographical biases \citep{shwartz-2022-good, bhatia-shwartz-2023-gd, bhatia-etal-2024-local, naous-xu-2025-origin, kruspe-2024-musical}.
Furthermore, models frequently struggle to bridge culture gaps \citep{tanwar-etal-2025-know, liu-etal-2024-multilingual, saha-etal-2025-reading, liu-etal-2026-tailored}.
To systematically quantify these cultural alignments, researchers frequently turn to global sociological instruments, such as the World Values Survey \citep{zhao-etal-2024-worldvaluesbench}. Yet, these evaluations suffer from the exact generative brittleness highlighted above. 
\citet{kabir-etal-2025-break} noted that evaluating cultural alignment in LLMs using standard multiple-choice surveys leads to wildly inconsistent outputs under minor perturbations, such as choice reordering. 

To bypass this brittleness, research has advocated moving toward unconstrained, open-ended generative evaluations. 
However, relying on open-ended generation introduces a new set of methodological bottlenecks. 
Specifically, it is notoriously difficult to reliably parse distinct categorical or ordinal scores from free-text outputs. 
Open-ended generation is highly susceptible to "fence-sitting", where models generate evasive disclaimers (e.g., "As an AI, I do not have a country...") rather than revealing their latent behavioral stance \citep{bhatia2025value, liu2025badworktimecrosscultural}. 
Dealing with these nuances typically requires utilizing another LLM as a judge to evaluate the text, which merely injects a secondary layer of bias into the evaluation pipeline.

\section{Methodology}
\label{sec:methods}

\begin{figure*}[t]
\centering
\resizebox{\textwidth}{!}{%
\begin{tikzpicture}[
  font=\scriptsize,
  box/.style={rounded corners, draw, thick, align=left, inner sep=6pt, text width=7.45cm},
  layer/.style={rounded corners, draw, thick, align=left, inner sep=6pt, text width=4.5cm},
  arrow/.style={-Latex, thick}
]

\node[box] (design) {%
\textbf{Experiment Design}\\
Fully crossed factorial design\\
Factors: \textit{Model} $\times$ \textit{Target Country}\\
Instrument: $K$ ordinal items (Likert 1--7)\\
Fixed system prompt + item templates
};

\node[box, below=0.45cm of design] (data) {%
\textbf{Data Collection}\\
Query each item under each condition\\
Record next-token PMF $P_{\mathrm{raw}}(t\mid x)$\\
Define valid token set $\mathcal{V}_{\mathrm{val}}$\\
};


\node[layer, below=0.65cm of data] (consensus) {%
\textbf{Consensus}\\
Map valid tokens to ordinal scale $\mathcal{Y}$\\
Compute multivariate Consensus/Dissension\\
};

\node[layer, left=0.35cm of consensus] (constraint) {%
\textbf{Constraint}\\
FailureRate $=1-\sum_{t\in\mathcal{V}_{\mathrm{val}}}P_{\mathrm{raw}}(t\mid x)$\\
Renormalize $P(t\mid x,\,t\in\mathcal{V}_{\mathrm{val}})$
};

\node[layer, right=0.35cm of consensus] (construct) {%
\textbf{Construct}\\
Convolve item PMFs $\Rightarrow$ composite score $S$\\
Compute main/interaction effect distributions\\
Summaries: $\mathbb{E}$, dPD, SNR
};

\node[draw, dashed, rounded corners, inner sep=2pt,
      fit=(constraint)(consensus)(construct)] (analysisbox) {};

\draw[arrow] (design) -- (data);
\draw[arrow] (data) -- (analysisbox);
\draw[arrow] (constraint) -- (consensus);
\draw[arrow] (consensus) -- (construct);

\end{tikzpicture}%
}
\caption{High-level exact-PMF framework: Fully crossed experiment design $\rightarrow$ token probability data collection $\rightarrow$ three-layer analysis (Constraint, Consensus, Construct).}
\label{fig:framework-overview}
\end{figure*}
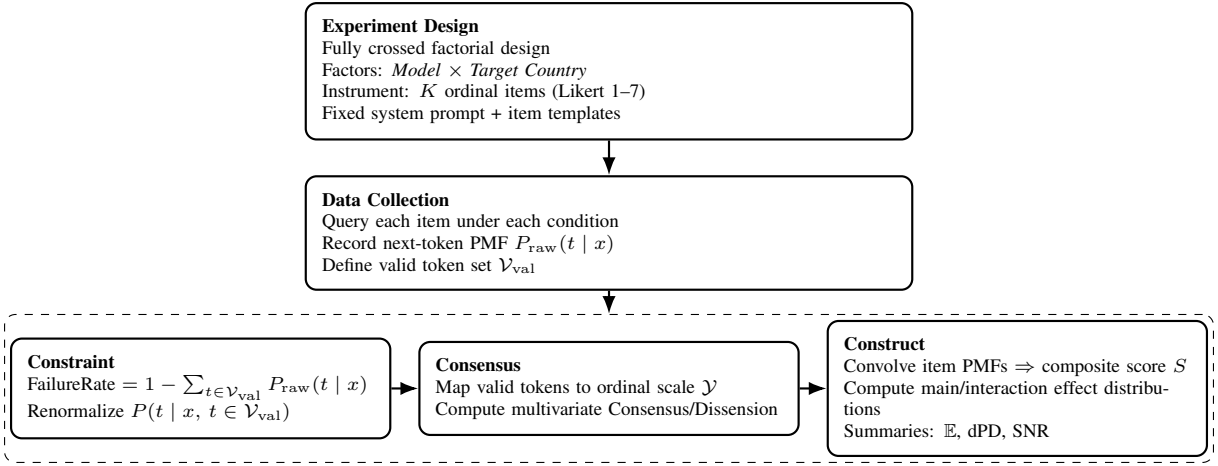

Figure \ref{fig:framework-overview} illustrates our framework. 
We formulate LLM evaluation as a fully crossed factorial experiment to resolve the design gap. 
Within this design, we treat the LLM as a generator of exact Probability Mass Functions (PMFs) to resolve the measurement gap and bypass generative sampling noise. 
Finally, we process exact PMFs through an analytical pipeline, which provides the foundation necessary for evaluating ordinal psychometric data and isolating causal behavioral shifts.

\subsection{Experiment design}

Rather than probing models with large datasets of ad-hoc prompt variations, we formalize LLM behavioral evaluation as a fully crossed fixed-effects experiment. 
We define our experiment through a set of independent factors (e.g., the \textsc{Model} being tested, and the \textsc{Target Country} in the prompt; see Fig. \ref{fig:prompt-setup}).  
Each factor contains distinct levels (e.g., four specific countries). 
A single experimental condition, denoted as $\boldsymbol{\lambda}$, represents one specific combination of these factors (e.g., \textsc{Model} = Llama-3.3, \textsc{Target} = Canada). 
Testing every possible combination of $\boldsymbol{\lambda}$ forms our complete experimental design space $\mathcal{D}$. 
By evaluating across the entirety of $\mathcal{D}$, we can mathematically isolate the effect of one variable while controlling for the others.
Under each condition $\boldsymbol{\lambda}$, the model evaluates a psychometric instrument consisting of $K$ individual questions (items). 
For each item, the model's response is bounded to an ordinal Likert scale $\mathcal{Y}$ ranging from $y_{\min}$ to $y_{\max}$ (e.g., digits 1 through 7). 

A central property of this design is context isolation. 
The model's internal state is effectively reset between items. 
This is an advantage that LLM behavioral analysis has over human behavioral analysis. 
With humans, there is a within-subjects carry-over effect that contaminates subsequent responses of human participants, whereas statelessness of LLMs does not suffer from this limitation.
For a fixed experimental condition $\boldsymbol{\lambda}$, the predicted response $Y_{k,\boldsymbol{\lambda}}$ for item $k$ depends solely on the specific isolated prompt $x_{k,\boldsymbol{\lambda}}$. 
Following \citet{wadi2025a}, we treat the items as conditionally independent given $\boldsymbol{\lambda}$, meaning the joint probability mass function over the scale factorizes exactly: $P(\mathbf{Y}_{\boldsymbol{\lambda}} = \mathbf{y} \mid \boldsymbol{\lambda}) = \prod_{k=1}^{K} P(Y_{k,\boldsymbol{\lambda}} = y_k \mid \boldsymbol{\lambda})$.

\subsection{Constraint}

Before analyzing behavior, we must ensure the model adheres to the psychometric instrument. LLMs generate predictions over a large vocabulary ($\mathcal{V}$), but a Likert scale only permits specific ordinal tokens (e.g., digits 1 through 7). 

For a given prompt $x$, let $P_{\mathrm{raw}}(t \mid x)$ denote the model's next-token probability mass function over the full vocabulary $\mathcal{V}$, where $t \in \mathcal{V}$ is a candidate token. 
Let $\mathcal{V}_{\mathrm{val}}\subset\mathcal{V}$ be the valid token set with a surjective mapping $\phi:\mathcal{V}_{\mathrm{val}}\to\mathcal{Y}$ associating valid tokens to ordinal values. 
To account for tokenizer idiosyncrasies (e.g., BPE leading spaces like \texttt{" 7"} versus \texttt{"7"}), $\mathcal{V}_{\mathrm{val}}$ includes all valid digit surface forms, but strictly excludes semantic equivalents (e.g., \texttt{"seven"}) to penalize syntactic constraint failures.
We define the model's failure rate as $1-\sum_{t\in\mathcal{V}_{\mathrm{val}}}P_{\mathrm{raw}}(t\mid x_{k,\boldsymbol{\lambda}})$, 
which quantifies how often the model fence-sits (i.e., refuses to answer) or hallucinates out of the constraint. 
For all subsequent behavioral analyses, we utilize the renormalized conditional distribution on $\mathcal{V}_{\mathrm{val}}$. Because multiple tokens can map to the same digit, we sum over the pre-image to yield an exact ordinal item PMF: 
$P(Y_{k,\boldsymbol{\lambda}}=y \mid \boldsymbol{\lambda})=\sum_{t \in \phi^{-1}(y)} P\!\left(t \mid x_{k,\boldsymbol{\lambda}},\,t\in\mathcal{V}_{\mathrm{val}}\right)$.
This follows the convention in human behavioral experiments,
which excludes human responses that do not adhere to instructions \citep{cronbach1950further}.

\subsection{Consensus}

Quantifying the expected value (mean) of a Likert response is necessary but insufficient, as identical means can mask fundamentally different behavioral regimes (e.g., polarizing mass on \{1, 7\} vs. stable mass on \{4, 4\}). 

Standard uncertainty metrics like Shannon entropy over valid tokens fail here because they are invariant to ordinal distance. 
To address this, we compute a \emph{multivariate} generalization of Consensus \citep{tastle2007consensus} directly on the model's PMF. This metric explicitly scales probabilities by their distance from the mean. Let $\boldsymbol{\mu}$ be the itemwise mean vector and $d_{\max}=\sqrt{K}\,(y_{\max}-y_{\min})$ be the maximum diagonal distance on the Likert scale. The multivariate Consensus ($\mathrm{Cns}$) is defined as\footnote{Dissension ($\mathrm{Dsn}$) is defined as $1 - \mathrm{Cns}$}:

{\small
\begin{equation}
\label{eq:consensus-multivariate}
\begin{aligned}
&\mathrm{Cns}(\mathbf{Y}_{\boldsymbol{\lambda}})
=
1+\sum_{\mathbf{y}\in\mathcal{Y}^K}
P(\mathbf{y})
\log_2\!\left(1-\frac{\|\mathbf{y}-\boldsymbol{\mu}\|_2}{d_{\max}}\right).
\end{aligned}
\end{equation}
}
Unlike entropy, this metric correctly quantifies the model's internal consistency. 
It severely penalizes polarization at the extremes and rewards concentrated probability mass, allowing us to accurately measure the model's opinion certainty (Table~\ref{tab:consensus-dissension-visual}).
It also resolves the issue of collapsing PMFs from the convolution operator (Fig.~\ref{fig:consensus-original-issue}).

\subsection{Construct}

A common approach to analyzing LLM behavior involves comparing descriptive statistics, such as average response scores across different prompt manipulations. 
While this practice is highly informative for observing general trends, analyzing raw point estimates inherently obscures information about the underlying variance. 
Consequently, it becomes difficult to ascertain whether the observed difference between groups is statistically significant, or whether the magnitude of the difference (effect size) is meaningful in practice.
Moreover, a key component of controlled experiments is understanding interactions between factors. 
Two factors can interact to amplify or decrease their individual effects \citep{wadi2026interplay}.

To overcome these limitations, we adapt the Analysis of Variance (ANOVA) framework for LLMs. 
A statistically grounded factorial framework allows us to systematically answer three critical questions: (1) whether the behavioral shift between conditions is statistically significant, (2) isolated from confounding variables, what the exact effect size and magnitude of that shift is, and (3) how experimental factors interact to jointly shape the model's behavior.

In standard psychometrics, a latent trait, such as ethnocentrism, cannot be reliably measured by a single question. 
Rather, it is quantified using an instrument containing a total of $K$ individual Likert items. 
Let $Y_{k,\boldsymbol{\lambda}}$ represent the model's ordinal response to item $k$ under a specific experimental condition $\boldsymbol{\lambda}$ (e.g., a specific target country and LLM). 
The composite construct score $S_{\boldsymbol{\lambda}}$ is defined as the sum of these individual item responses ($S_{\boldsymbol{\lambda}} = \sum_{k=1}^K Y_{k,\boldsymbol{\lambda}}$). 
Because our framework operates on exact probability distributions rather than single scalar outputs, we evaluate the full probability mass function (PMF) of this composite score.
 
Given the conditional independence of items (Sec. \ref{sec:independence-assumption}), according to standard probability theory, the distribution of a sum of independent variables is computed via discrete convolution (denoted by the operator $\circledast$). Therefore, the exact probability distribution of the total construct score $P_{S_{\boldsymbol{\lambda}}}$ is derived by convolving the individual item PMFs: $P_{S_{\boldsymbol{\lambda}}}=P_{Y_{1,\boldsymbol{\lambda}}}\circledast\cdots\circledast P_{Y_{K,\boldsymbol{\lambda}}}$.
Unlike standard evaluation methods that compute descriptive statistics over a handful of noisy, sampled text generations, 
the convolution approach analytically derives the exact probability distribution of the final construct score ($P_{S_{\boldsymbol{\lambda}}}$). 
This propagates all aleatoric uncertainty from the token level up to the final construct level.

Next, we need to systematically decompose LLM responses ($P_{S_{\boldsymbol{\lambda}}}$) into distinct, interpretable components. 
Rather than evaluating each experimental condition ($\boldsymbol{\lambda}$) as a whole, 
we decompose the model's overall construct distribution into three parts: 
the model's baseline behavior, 
the isolated Main Effects (e.g., how much changing the Target Country specifically shifts the behavior), 
and the Interaction Effects (e.g., how the Target Country interacts with a specific LLM). 
We formalize this using a distributional Hoeffding decomposition. 
Let $S_0$ represent the grand-mixture baseline across all conditions, $E_c(\lambda_c)$ represent the Main Effect distribution for factor $c$ at level $\lambda_c$, and $E_U(\boldsymbol{\lambda}_U)$ represent the higher-order Interaction Effect distribution for a subset of factors $U$. 
This decomposition satisfies the following identity in expectation:

\begin{equation}
\label{eq:anova-identity-main}
\small
\mathbb{E}[S_{\boldsymbol{\lambda}}] = \underbrace{\mathbb{E}[S_0]}_{\text{Grand Mean}} + \sum_{c \in \mathcal{C}} \underbrace{\mathbb{E}[E_c(\lambda_c)]}_{\text{Main Effects}} + \sum_{\substack{U \subseteq \mathcal{C} \\ |U| \ge 2}} \underbrace{\mathbb{E}[E_U(\boldsymbol{\lambda}_U)]}_{\text{Interactions}},
\end{equation}

This exact-PMF provides the guarantee that the expected values of the derived probability distributions ($\mathbb{E}[E_c(\lambda_c)]$ and $\mathbb{E}[E_U(\boldsymbol{\lambda}_U)]$) match the proven fixed-effects parameters calculated in traditional human ANOVA experiments. 
We formalize this alignment with the following theorem (Proof in App. \ref{cor:main-effect-anova}):
\begin{theorem}
\label{thm:anova-consistency-main}
For any fully crossed experimental design, the expectations of the isolated effect distributions ($E_c$, $E_U$) recover the classical unique Hoeffding decomposition.
\end{theorem}
In practical terms, Theorem \ref{thm:anova-consistency-main} guarantees that when we measure a higher-order behavior, such as a "country-of-origin bias" (an interaction effect), we have mathematically stripped away the model's baseline ethnocentrism (main effect). 
This ensures researchers do not double-count biases, achieving the exact same control over confounding variables used by human behavioral scientists.

Finally, to translate these effect distributions into interpretable behavioral insights, we extract the following summary metrics for our case study, reported for any given effect distribution $E$:
\begin{itemize}
    \item \textbf{Mean Shift ($\mathbb{E}[E]$):} The expected value of the distribution, representing the actual magnitude (effect size) of the isolated behavioral effect.
    \item \textbf{Predictive Dispersion ($\mathrm{SD}(E)$):} The standard deviation of the effect, capturing the model's aleatoric uncertainty regarding the shift.
    \item \textbf{Discrete Probability of Direction ($\mathrm{dPD}$):} Defined as $\max(\mathbb{P}(E > 0), \mathbb{P}(E < 0))$, this metric quantifies the extent to which the behavioral shift exhibits a consistent directionality (i.e., strictly positive or strictly negative), independent of the effect's magnitude (App. \ref{app:dpd}).
    \item \textbf{Signal-to-Noise Ratio ($\mathrm{SNR}$):} Calculated as $|\mathbb{E}[E]|/\mathrm{SD}(E)$, SNR quantifies the statistical reliability of the effect, indicating whether the observed behavioral shift is a robust, distinct phenomenon or merely an artifact of the model's internal uncertainty.
\end{itemize}
With these metrics, we now apply this methodology to our main sociotechnical concern. 
We test whether LLMs exhibit systematically biased ethnocentric behaviors based on their country of origin.
See App. \ref{app:metric-interpretation} for details of each metric.

\section{Experiment}

\subsection{Design}

We test our framework using the Consumer Ethnocentrism Tendencies Scale (CETSCALE; \citealp{shimp1987consumer}). 
This is a widely validated $K{=}17$-item Likert instrument with a response space of $\mathcal{Y}=\{1,\dots,7\}$. While a 17-item survey is much smaller than typical NLP benchmark datasets, it constitutes a complete, rigorously tested measurement tool in the behavioral sciences. 
Using this exact, unmodified scale ensures that we are reliably measuring the intended bias without diluting the test. 

We employ a symmetrical, fully crossed $5\times 4$ fixed-effects design across factors $\{\textsc{Model}, \textsc{Target}\}$.
The model levels include open-weights models representing four distinct countries-of-origin:
$\mathcal{L}_{\textsc{Model}}=\{$Llama~3.3~70B and Gemma~3~27B (USA); Qwen3~Next~80B (China); Aya~Expanse~32B (Canada); and Ministral~14B (France)$\}$.
The Target Country levels mirror these origins:
$\mathcal{L}_{\textsc{Target}}=\{\text{USA},\text{China},\text{Canada},\text{France}\}$.
This symmetrical $5 \times 4$ design is critical as it allows our framework to mathematically isolate (i) the global main effects of $\textsc{Model}$ and $\textsc{Target}$, and (ii) the precise $\textsc{Model}\times\textsc{Target}$ interaction effects (e.g., asking a Chinese-developed model to express its ethnocentrism towards USA), which represent the actual country-of-origin bias.

All conditions share an identical stimulus structure. A fixed system message strictly enforces the response schema ("...Respond with ONLY the digit."). 
Each of the 17 CETSCALE items acts as a user-message template ($T_k$) containing a target country placeholder (e.g., \{People\}, \{Adj\}). 
The \textsc{Target} factor deterministically instantiates these placeholders via a strict lexicon (Figure \ref{fig:prompt-setup}). 
Importantly, prompt wording is treated as a fixed experimental control rather than tuned per model, with the aim to mirror stimulus control in human behavioral experiments.

For each $(\textsc{Model},\textsc{Target},k)\in \mathcal{L}_{\textsc{Model}}\times\mathcal{L}_{\textsc{Target}}\times\{1,\dots,17\}$, we perform a single forward pass and record the exact next-token PMF ($P_{\mathrm{raw}}(t\mid x_{k,\boldsymbol{\lambda}})$) over the full vocabulary.\footnote{Because we directly extract the analytical probability distribution over the logits, our measurement is strictly deterministic. It bypasses autoregressive decoding entirely, rendering generation hyperparameters such as temperature or top-
p irrelevant.}
We define $\mathcal{V}_{\mathrm{val}}$ as the specific valid digit tokens $\{1,\dots,7\}$ for each model's respective tokenizer. 
In total, this yields $5\times 4\times 17 = 340$ exact item-level predictive distributions. These PMFs are directly passed through the Constraint Layer to measure formatting failures, the Consensus Layer to measure ordinal certainty, and the Construct Layer to isolate distributional Hoeffding effect sizes.

\subsection{Results}

\paragraph{Constraint.}
As shown in Figure \ref{fig:failure-rate}, Gemma 3, Llama 3.3, and Qwen3 demonstrate near-perfect adherence to the task constraints, concentrating their probability mass almost entirely on valid ordinal tokens ($M<0.001$, $\mathrm{SD}<0.001$). 
In contrast, Ministral exhibits a small but consistent failure to constrain its mass to the requested digits ($M=0.002, \mathrm{SD}=0.001$). More notably, Aya Expanse displays the highest aggregate failure rate ($M=0.006$) driven by large variance ($\mathrm{SD}=0.049$). This indicates that utilizing Aya or Ministral in automated agentic settings could lead to sporadic breakdowns due to their probabilistic inability to adhere strictly to psychometric constraints.

\begin{figure}[htbp]
  \centering
    \centerline{\includegraphics[width=\columnwidth]{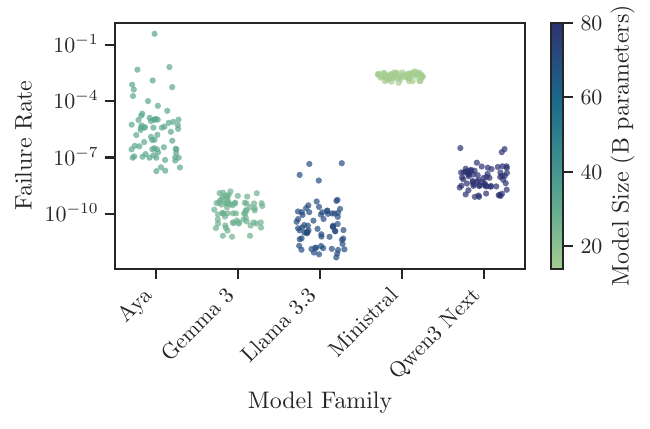}}
    \caption{Failure Rate across the evaluated LLMs.}
    \label{fig:failure-rate}
    \vspace{-3mm}
\end{figure}

\paragraph{Consensus.}
\label{sec:exp-consensus}
Across all target countries, four models show exceptionally high internal agreement regarding the instrument, maintaining $\mathrm{Cns}>0.88$ (Figure \ref{fig:cns}). 
Ministral is a major outlier, exhibiting substantially lower consensus across all targets ($\mathrm{Cns}\approx 0.65$--$0.67$). This shows that, even after enforcing valid digits, Ministral's item-level PMFs frequently place large probability mass on conflicting ends of the Likert scales. 
This stark contrast highlights the critical necessity of an ordinal-aware dispersion metric. Standard entropy would fail to capture this specific behavioral dissonance (Table \ref{tab:cns-entropy}).

\begin{figure}[htbp]
    \centering
    \centerline{\includegraphics[width=\columnwidth]{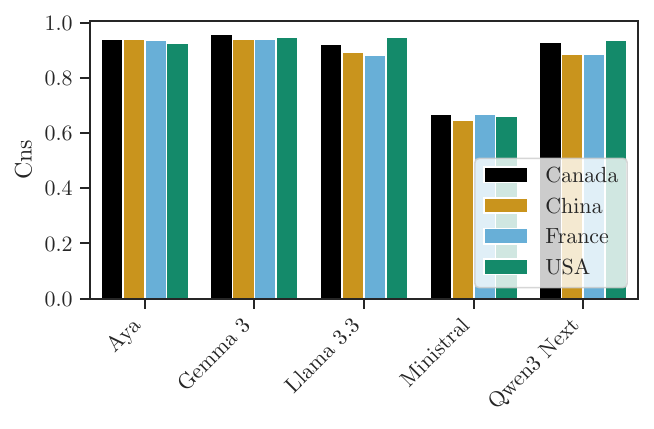}}
    \caption{Multivariate Consensus ($\mathrm{Cns}$) for the CETSCALE instrument. Lower values indicate severe behavioral polarization.}
    \label{fig:cns}
    \vspace{-3mm}
\end{figure}

\begin{figure}[ht]
    \centering
    \begin{subfigure}{0.9\columnwidth}
    \centering
      \includegraphics[width=0.9\columnwidth]{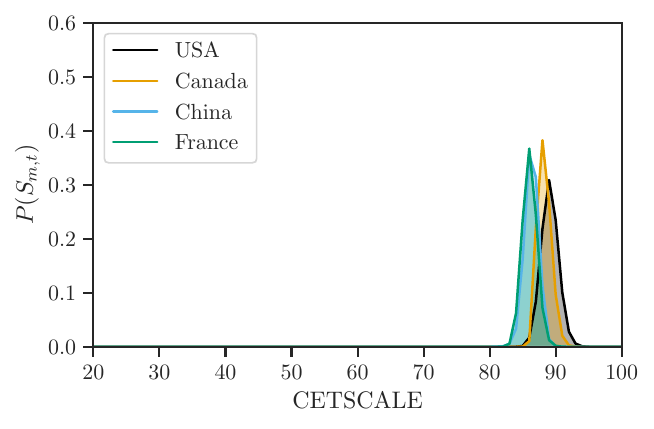}
      \caption{Aya Expanse 32B}
    \end{subfigure}
    \begin{subfigure}{0.9\columnwidth}
    \centering
      \includegraphics[width=0.9\columnwidth]{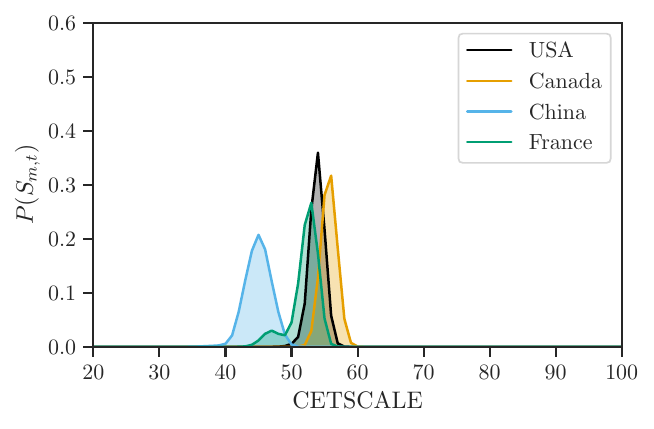}
      \caption{Qwen3 Next 80B}
    \end{subfigure}
    \caption{Distribution of CETSCALE for Aya Expanse 32B (high) vs. Qwen3 Next 80B (low) ethnocentrism and Target Country.}
    \label{fig:specific-pmf-2}
    \vspace{-3mm}
\end{figure}

\paragraph{Construct.}
To contextualize the LLMs' response to the CETSCALE, 
we compare each model's predictive distribution for \textsc{Target}=USA against the human sample means and standard deviations reported in the original CETSCALE validation studies (\citealp{shimp1987consumer}; Table \ref{tab:cetscale-human-llm-comparison}; Figure~\ref{fig:human-llm})\footnote{The human data (collected prior to 1987) serves strictly as a historical magnitude anchor to calibrate the CETSCALE, not as a baseline for modern human attitudes. Note that this reference concerns expected-score magnitude only: a model's PMF dispersion reflects the internal uncertainty of a single entity and is not commensurable with the interpersonal variance of a human population sample."}. 
Several LLMs exhibit expected composite scores comparable to, or explicitly exceeding, those observed in the most ethnocentric human populations. 
See the full distributions in Figure~\ref{fig:specific-pmf}.

\begin{table}[htbp]
  \caption{Comparison of CETSCALE scores between historical human populations \citep{shimp1987consumer} and LLMs evaluated in this study (restricted to the \textsc{Target}=USA condition). Higher composite score indicates higher ethnocentrism.}
  \label{tab:cetscale-human-llm-comparison}
  \centering
    \begin{small}
      \begin{sc}
        \begin{tabular}{llrr}
          \toprule
          Subject & Sample / Model & $\mathbb{E}[S]$ & SD[S] \\
          \midrule
          \multirow[t]{6}{*}{Human} 
          & Detroit (USA) & 68.58 & 25.96 \\
          & Carolinas (USA) & 61.28 & 24.41 \\
          & Denver (USA) & 57.84 & 26.10 \\
          & Los Angeles (USA) & 56.62 & 26.37 \\
          & Students (Pre) & 51.92 & 16.37 \\
          & Students (Post) & 53.39 & 16.52 \\
          \midrule
          \multirow[t]{5}{*}{LLM} 
          & Aya Expanse 32B & 89.11 & 1.32 \\
          & Llama 3.3 70B & 71.99 & 0.95 \\
          & Ministral 14B & 70.20 & 5.25 \\
          & Gemma 3 27B & 60.32 & 0.94 \\
          & Qwen3 Next 80B & 53.85 & 1.17 \\
          \bottomrule
        \end{tabular}
      \end{sc}
    \end{small}
    \vspace{-3mm}
\end{table}

\paragraph{Main Effects.}
Extracting precise main effect distributions (Table \ref{tab:main-effects}), reveals large and systematic behavioral divides. 
\textsc{Aya-32B} exhibits a severe positive deviation from the grand mean ($\mathbb{E}=21.19$, $\mathrm{SNR}=1.63$, $\mathrm{dPD}>0.99$), indicating a pervasive ethnocentrism that far exceeds human baselines. 
Conversely, \textsc{Qwen3-80B} demonstrates a highly robust negative main effect ($\mathbb{E}=-14.63$, $\mathrm{SNR}=1.20$, $\mathrm{dPD}=0.93$), resulting in expected scores similar to liberal human populations.

Analyzing the main effects for the Target Countries (Table \ref{tab:main-effects}; and the pairwise contrasts in Table \ref{tab:country-contrasts}; Figure~\ref{fig:country-main}) reveals that the most statistically robust biases involve China. 
Across all models, there is a systemic decrease in ethnocentrism when the target is China ($\mathbb{E}=-6.46$, $\mathrm{SNR}=1.04$), contrasting sharply with the relative favoritism shown toward North American targets (e.g., USA-China contrast $\mathbb{E}=9.30$, $\mathrm{SNR}=1.34$).

\begin{table}[htbp]
  \caption{Summaries of main effects.}
  \label{tab:main-effects}
  \centering
    \begin{small}
      \begin{sc}
\begin{tabular}{lrrll}
\toprule
Parameter & $\mathbb{E}$ & $\mathrm{SD}$ & $\mathrm{SNR}$ & $\mathrm{dPD}$ \\
\midrule
$\mu_{\emptyset}$ & 66.25 & 14.21 &  &  \\
Aya Expanse 32B & 21.19 & 13.01 & 1.63 & 0.99 \\
Gemma 3 27B & -11.99 & 13.15 & 0.91 & 0.77 \\
Llama 3.3 70B & 0.31 & 12.79 & 0.02 & 0.55 \\
Ministral 14B & 5.11 & 9.59 & 0.53 & 0.62 \\
Qwen3 Next 80B & -14.63 & 12.17 & 1.20 & 0.93 \\
USA & 2.84 & 5.43 & 0.52 & 0.55 \\
Canada & 4.62 & 5.72 & 0.81 & 0.90 \\
China & -6.46 & 6.19 & 1.04 & 0.95 \\
France & -1.00 & 4.72 & 0.21 & 0.54 \\
\bottomrule
\end{tabular}
      \end{sc}
    \end{small}
    \vspace{-3mm}
\end{table}

\paragraph{Interaction Effects and Country-of-Origin Bias.}
We define country-of-origin bias as a positive interaction between a model and its own country of origin, measured as an isolated deviation from the grand-mean baseline. 
The interaction distributions ($E_U$) show that several models exhibit structured, non-additive interaction patterns (Figure~\ref{fig:interactions}; Table~\ref{tab:coefficient-table-marginal}), though the strength and directionality of these effects vary considerably across models.
The clearest regional pattern appears for \textsc{Gemma 3-27B} and \textsc{Llama 3.3-70B} (both U.S.-developed), which show positive interactions toward the North American bloc (USA and Canada) alongside a pronounced negative interaction with China. 
Both models exhibit positive own-country (USA) interactions ($+3.21$ and $+2.59$), consistent with ingroup favoritism, although these particular ingroup cells are directionally moderate ($\mathrm{dPD}=0.63$ and $0.58$).
We emphasize that clean own-country favoritism is not universal across all five models. 
\textsc{Aya-32B} (Canada, $-3.77$) and \textsc{Qwen3 Next-80B} (China, $-0.20$) do not exhibit a positive own-country interaction, and \textsc{Ministral-14B} shows a positive but modest France interaction ($+2.45$) that is smaller than its interaction with China.
This showcases how the PMF ANOVA framework isolates each interaction from the model's baseline behavior and main effects and allows researchers to empirically test whether a construct such as country-of-origin bias is present, absent, or reversed for a given model.

\begin{figure}[htbp]
  \centering
  \centerline{\includegraphics[width=\columnwidth]{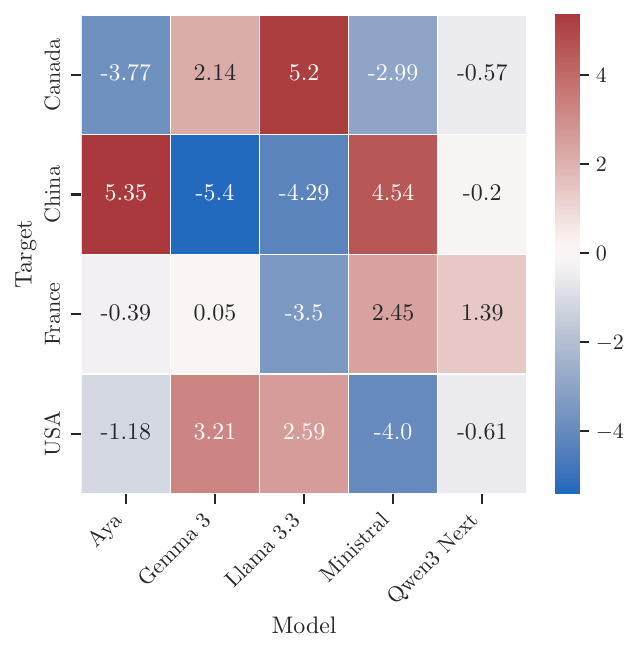}}
    \caption{Interactions of Model and Target Country.}
    \label{fig:interactions}
    \vspace{-3mm}
\end{figure}

\paragraph{Aggregate vs. Factorial Benchmarks.}
To empirically demonstrate the danger of conflating causal mechanisms, we compared factorial interactions against an aggregate benchmark (App.~\ref{tab:aggregate-vs-factorial}). 
If a researcher were to evaluate Target Country bias by taking an aggregate average, they would conclude there is a universal negative bias against France ($\mathbb{E}=-1.00$) and China ($\mathbb{E}=-6.46$). 
However, factorial decomposition shows this aggregate view is directionally incorrect for specific models: Ministral-14B actually exhibits positive ingroup favoritism toward France ($+2.45$). 
Similarly, while the aggregate benchmark suggests a universal penalization of China, both Aya-32B ($+5.35$) and Ministral-14B ($+4.54$) exhibit strong positive interactions. 

\subsection{The Cost of Generative Sampling}
\label{sec:sampling-cost}

To quantify the cost of the standard NLP Monte Carlo sampling paradigm, we benchmarked a simulated sampling approach against our exact distributions (acting as ground truth; App.~\ref{app:sampling-cost}). 
Standard text sampling introduces severe, previously unquantified reliability issues in behavioral evaluation. 

First, sampling noise degrades the detection of subtle biases. 
Because the standard error decays only as $1/\sqrt{N}$, a standard sampling budget of $N=100$ still yields a standard error of up to 0.52 composite-score points for low-consensus models (App.~\ref{app:sampling-noise}). 
To determine how often this noise leads to false conclusions, we calculated the flip probability ($P_{\text{flip}}$): the chance that finite sampling recovers an effect with the wrong sign (App.~\ref{app:flip-analysis}). 
At $N=10$, standard sampling flips the sign of small effects ($|E| < 1$) 18\% of the time, and even at $N=100$, it flips them 6.3\% of the time. 
The exact-PMF framework achieves the $N \to \infty$ limit in a single forward pass, ensuring no interaction effect is masked or reversed by sampling variance.
Second, non-default decoding (e.g., temperature scaling) introduces systematic biases that no amount of sampling can fix. 
Our analysis (App.~\ref{app:decoding-bias}) shows that for models with low internal consensus (e.g., Ministral), modifying temperature artificially shifts the latent construct score by up to 2.79 points. 

\subsection{Prompt Sensitivity}
\label{sec:prompt-sensitivity}

A common critique in NLP behavioral evaluation is that model outputs are highly sensitive to prompt phrasing. 
A key strength of casting evaluation as a fully crossed factorial design is its extensibility. 
Any potential confounder can be introduced as an additional factor and isolated via the distributional ANOVA.

To demonstrate this, we promoted the administration (system) prompt to a first-class experimental factor ($\lambda_{\text{Prompt}}$). 
As detailed in App.~\ref{app:prompt-framing}, we tested five additional framing variants (ranging from a market-research persona to minimal formatting; Table~\ref{tab:prompt-variants}) across three models and four target countries, yielding 1,224 exact item-level PMFs.

The full ANOVA decomposition (Tables~\ref{tab:prompt-effects} and \ref{tab:prompt-effects-p2}) show three key findings. 
First, the core country-of-origin signals are robust. 
The \textsc{Target} main effects reproduce the directional pattern of the main study under statistical control for framing. 
Second, prompt framing acts as a causal main effect of its own, shifting the latent construct score by up to 13.5 points depending on the phrasing used. 
Third, the framework formally quantifies model-specific prompt sensitivity via the $\textsc{Model} \times \textsc{Prompt}$ interaction. 
We isolate that while Gemma~3 exhibits low sensitivity to wording changes, both Aya and Ministral interact strongly with specific framings.

\section{Conclusion}

In this work, we introduced a new paradigm for the behavioral evaluation of LLMs, shifting the focus from large-scale, unstructured benchmark datasets to controlled factorial experiments. 
By bridging human psychometrics with LLM mechanics, we addressed critical gaps in experimental design, measurement, and analysis. 
By operating directly on exact, token-level Probability Mass Functions (PMFs) without generative sampling noise, we formulated an analytical pipeline, featuring a multivariate ordinal consensus metric and a distributional ANOVA. This pipeline allows researchers to mathematically isolate causal behavioral shifts and factor interactions without artificially inflating variance.

We validated this framework with a case study on consumer ethnocentrism. 
We isolated systemic country-of-origin biases that would otherwise be obscured by confounding variables in standard observational datasets. 
Ultimately, by prioritizing causal experimental control and exact probability distributions over text sampling, this general-purpose framework provides a statistically rigorous foundation for future research into LLM alignment, values, and latent behavioral traits.

\section*{Limitations}

While our mathematical framework is designed to be applied to any behavioral evaluation, our empirical experiments focus exclusively on consumer ethnocentrism. We deliberately chose to analyze a single construct in depth to clearly illustrate how the distributional ANOVA operates under the hood. Now that the methodology is established, applying this exact-PMF framework to a much wider array of psychological traits, moral values, and geopolitical biases is an exciting direction for future work. Therefore, the specific findings in this paper should not be interpreted as comprehensive claims about the general political ideology or cultural alignment of the tested models.

Furthermore, all experiments were conducted in English. 
Whether the observed systemic biases hold when models are queried in other languages or scripts remains an open empirical question. 
Additionally, as these are snapshot evaluations, our findings reflect the models' latent states at the time of testing and may not generalize to future weight updates or subsequent model versions.

A common critique in NLP evaluation concerns model sensitivity to prompt variations. 
We adopt the strict position from behavioral psychometrics that the prompt template is a standardized measurement instrument, not a tunable hyperparameter. 
In human surveying, rewording a validated item constitutes a fundamental change to the construct being measured, rather than a robustness failure of the experiment. 
We stress, however, that our framework treats any experimental variable as a crossed factor, so prompt phrasing is itself a first-class factor $\lambda_{\text{Prompt}}$ that the design directly supports. 
Introducing a set of semantically-equivalent paraphrases as levels of $\lambda_{\text{Prompt}}$ requires no change to the underlying theory. 
The Prompt main effect and the Model $\times$ Prompt and Target $\times$ Prompt interactions are recovered by the identical distributional-ANOVA method. 
In this view, prompt sensitivity becomes a measurable, isolable effect rather than an uncontrolled confound.

Methodologically, our framework relies on exact next-token probability distributions and is explicitly designed for constrained, single-token response paradigms (e.g., Likert scales). 
While major proprietary providers (e.g., OpenAI, Anthropic, Google) currently expose log-probabilities via their APIs, providers that entirely obscure their next-token probabilities fall outside the scope of this approach. 
Furthermore, as LLMs increasingly utilize multi-token reasoning protocols (e.g., Chain-of-Thought) prior to generating a final response, extending our Consensus and Construct statistics to integrate over latent multi-token trajectory spaces remains a complex, open challenge for future work.

\section*{Ethical Considerations}

This framework is intended as a diagnostic tool for model evaluation and behavioral research.
Scores on any single psychometric construct should not be interpreted as a comprehensive safety certification, 
nor should findings about specific models be generalized to the organizations or cultures associated with their development.
We caution that comparison to human benchmarks is provided for calibration purposes only and should not be used to 
normalize or excuse systematic bias in language models.
Finally, as with any measurement framework, there is a risk that optimizing directly against these metrics 
could produce models that score favorably without genuinely reducing the underlying biases the instruments are designed to detect.

\section*{Acknowledgements}

The authors acknowledge the use of AI writing assistants for proofreading and improving the overall clarity of the writing.


\bibliography{custom,custom_2}

\appendix
\section{Appendix}
\label{sec:appendix}

\subsection{Experiment Design}
\label{app:experiment-design}

We restate the experimental design and notation used throughout the paper
in a self-contained format here and provide the proofs.

We adopt a generalized fully crossed factorial experimental design.
All independent variables, including language model identity, are treated as factors embedded within a unified design space.

Let $\mathcal{C} = \{1,\dots,C\}$ denote the index set of experimental factors.
Each factor $c \in \mathcal{C}$ admits a finite set of discrete levels, with cardinality $\Lambda_c \in \mathbb{N}$.
We denote the level selected for factor $c$ by
\[
\lambda_c \in \{1,\dots,\Lambda_c\}.
\]
 
The complete experimental design is given by the Cartesian product of the level sets of all factors,
\begin{equation}
\label{eq:design-space}
\mathcal{D}
=
\prod_{c \in \mathcal{C}}
\{1,\dots,\Lambda_c\}.
\end{equation}

For any subset of factors $U \subseteq \mathcal{C}$, we denote
$\boldsymbol{\lambda}_U := (\lambda_c)_{c \in U}$ the vector of factor levels
indexed by $U$.
The vector of the levels for the complete factor set is hence denoted as 
$\boldsymbol{\lambda}_{\mathcal{C}}$ 
(abbreviated as
$\boldsymbol{\lambda}$).

The dependent variable is measured through a psychometric instrument consisting of $K$ items. 
Each item is designed to elicit a response on a shared ordinal scale formalized as an ordered set $(\mathcal{Y}, \preceq)$
\begin{equation}
    \label{eq:Y-support}
    \mathcal{Y} = \{y_{\min}, \dots, y_{\max}\} .
\end{equation}
Each prompt instructs the model to respond with a single digit only (one token) indicating the ordinal category, $\mathcal{Y}$, with no additional text. We therefore treat the model's response as the next-token distribution.

The joint response from all items, $K$, corresponds to a vector 
\begin{equation}
    \label{eq:y-vector}
    \mathbf{y} = (y_1, \dots, y_K) \in \mathcal{Y}^K.
\end{equation}

In behavioral experiments only a subset of the complete output space
corresponds to valid responses. 
Let $\mathcal{V}$ denote the model's full token vocabulary, and let $\mathcal{V}_{\text{val}} \subset {\mathcal{V}}$ denote the set of valid tokens for the task {with the same cardinality as $\mathcal{Y}$, $|\mathcal{V}_{\text{val}}|=|\mathcal{Y}|$}. 
Let $P_{\text{raw}}(t \mid x)$ denote the model's next-token probability distribution given a prompt $x$, defined over the full vocabulary {$\mathcal{V}$}. We define the valid probability mass as
\begin{equation}
    P_{\text{val}}(x) = \sum_{t \in \mathcal{V}_{\text{val}}} P_{\text{raw}}(t \mid x).    
\end{equation}

To interface with the language model, we define a surjective mapping 
\begin{equation}
    \label{eq:phi}
\phi: \mathcal{V}_{\text{val}} \to \mathcal{Y}    
\end{equation}
that accounts for tokenizer-specific variations of digits (e.g., leading spaces) and associates a specific subset of valid vocabulary tokens $\mathcal{V}_{\text{val}} \subset {\mathcal{V}}$ with their numeric values, $\mathcal{Y}$.
This mapping transforms the token distribution into a probability mass function over ordinal integers, allowing us to use ordinal dispersion measures to calculate the internal consistency of the model's predictive distribution.

\subsection{Conditional Independence via Context Isolation}
\label{sec:independence-assumption}

A central property of this design is that the model's internal state is reset between items. For a fixed experimental condition $\boldsymbol{\lambda}$, the response $Y_{k,\boldsymbol{\lambda}}$ for item $k$ depends solely on the specific prompt $x_{k,\boldsymbol{\lambda}}$ and the model parameters. 
Following \citet{wadi2025a}, we treat the items as conditionally independent given $\boldsymbol{\lambda}$.
Formally, for any realization vector $\mathbf{y} \in \mathcal{Y}^K$, the joint probability mass function factorizes as
\begin{equation}
\label{eq:joint-distribution-app}
    P(\mathbf{Y}_{\boldsymbol{\lambda}} = \mathbf{y} \mid \boldsymbol{\lambda}) = \prod_{k=1}^{K} P(Y_{k,\boldsymbol{\lambda}} = y_k \mid \boldsymbol{\lambda}).
\end{equation}
This assumption justifies the use of product measures for joint distributions (Appendix \ref{app:consensus}) and convolution for composite sums (Appendix \ref{app:construct}).

\subsection{The Consensus Layer}
\label{app:consensus}
Having confirmed the model's focus on valid outputs, 
we now superimpose the numeric, ordinal structure of the task 
onto the analysis. 
We require a measure that quantifies whether the model's attitude towards the numeric scale is behaviorally consistent. Consider a single item with a $7$-point Likert scale ($\mathcal{Y}=\{1,\dots,7\}$). A distribution assigning equal mass to opposing endpoints with the following probability mass functions (PMFs)
\[
P(Y=1)=0.5, \quad P(Y=7)=0.5,
\]
exhibits dissension, as the model simultaneously expresses diametrically opposing views. In contrast, a distribution concentrated on adjacent scale points,
\[
P(Y=5)=0.5, \quad P(Y=6)=0.5,
\]
is logically consistent, reflecting only local uncertainty regarding the precise intensity of the attitude.

\begin{table}[t]
  \caption{Visualization of LLM Consensus (Cns) and Dissension (Dsn) cases. Each histogram depicts a characteristic output distribution on a 1--7 scale.}
  \label{tab:consensus-dissension-visual}
  \begin{center}
    \begin{small}
      \begin{sc}
        \begin{tabular}{ccm{2.7cm}}
          \toprule
          Cns & Entropy & \multicolumn{1}{c}{Distribution} \\
          \midrule
          1.00 & 0.00 & \includegraphics[width=\linewidth]{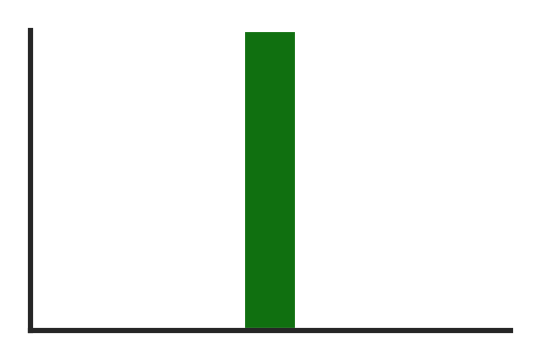} \\
          
          0.84 & 1.57 & \includegraphics[width=\linewidth]{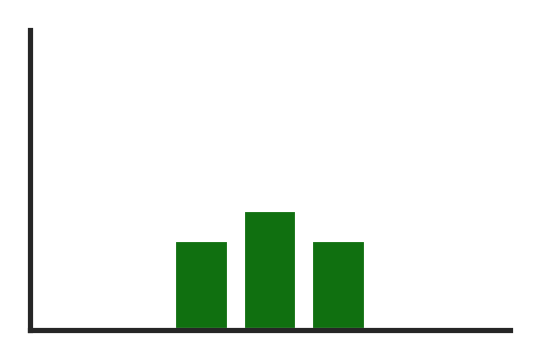} \\
          
          0.47 & 2.81 & \includegraphics[width=\linewidth]{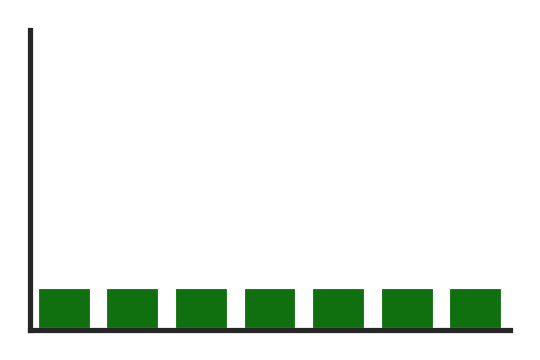} \\
          
          0.74 & 1.00 & \includegraphics[width=\linewidth]{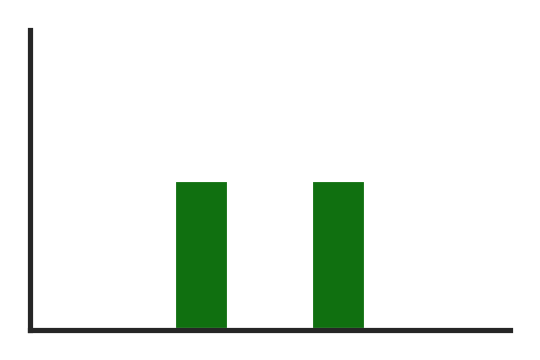} \\

          0.00 & 1.00 & \includegraphics[width=\linewidth]{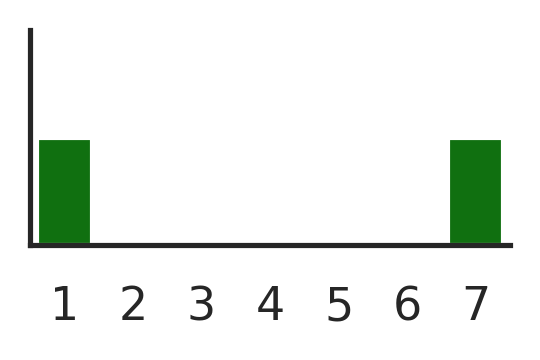} \\
          \bottomrule
        \end{tabular}
      \end{sc}
    \end{small}
  \end{center}
  \vskip -0.1in
\end{table}

To obtain a measure of consensus from a sample of individuals, 
\citet{tastle2007consensus} introduced Consensus ($\mathrm{Cns}$) as
\begin{equation}
    \small 
    \label{eq:consensus-original}
        \mathrm{Cns}(Y) = 1 + \sum_{y \in \mathcal{Y}} \biggl[ P(Y=y) \times \log_2 \left( 1 - \frac{|y - \mu_Y|}{d_{\max}} \right) \biggr]
\end{equation}
where $\mu_Y = \sum_{y \in \mathcal{Y}} y P(Y=y)$ is the expected value of the ordinal distribution and $d_{\max} =  y_{\max} - y_{\min}$.
Table \ref{tab:consensus-dissension-visual} illustrates Consensus on different PMFs.
The issue with using entropy to measure consensus is shown in the 4th 
($P(3)=0.5$ and $P(5)=0.5$) 
and 5th row 
($P(1)=0.5$ and $P(7)=0.5$), 
where the latter shows lower levels of consensus, but entropy gives the same value
($Entropy=1.00$).

Whereas the original Consensus statistic is defined over a sample
of aggregated scalar responses,
our setting differs in two fundamental ways.
First, with language models as subjects, we have direct access to the PMF
of the data generation process rather than sampled responses.
Second, behavioral instruments such as Likert scales typically consist of
multiple items, each with varying levels of Consensus and Dissension.
Calculating Consensus over aggregated responses
disregards the inter-item Dissension.

\begin{figure}[t]
    \includegraphics[width=\columnwidth]{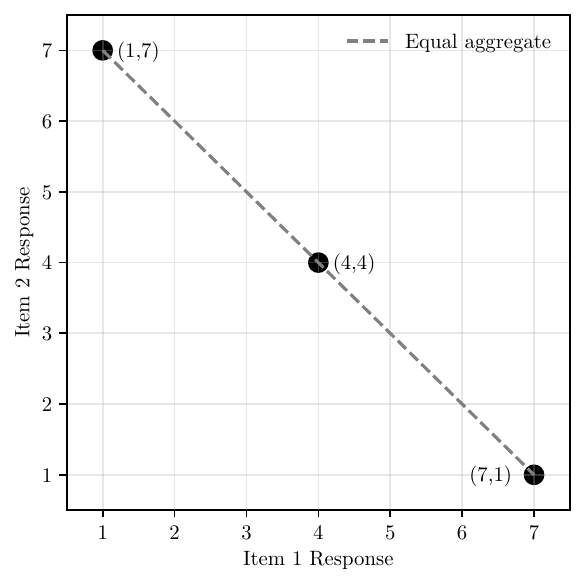}
    \caption{
    The failure of aggregate Consensus in capturing between-item dissension.
    Distinct response patterns in a two-item Likert space
    (e.g., $(4,4)$ vs.\ $(1,7)$)
    collapse to the same scalar score after aggregation.
    The original Consensus measure operates on this projection
    and cannot detect between-item dissension.
    }
    \label{fig:consensus-original-issue}
\end{figure}

Standard aggregation procedures, 
such as summation or averaging,
project this space onto a single scalar dimension. 
Figure~\ref{fig:consensus-original-issue} illustrates this effect for a two-item, 7-point Likert scale. Response patterns such as $(4,4)$, $(1,7)$, and $(7,1)$ all yield the same aggregate sum, causing the inter-item Dissension to be lost. However, these points occupy very different locations in the original two-dimensional response space, with $(1,7)$ and $(7,1)$ representing maximal disagreement between items.

Thus, we derive the Multivariate Consensus (Consensus from now on) across all items on a common ordinal scale. 
For each item $k \in \{1,\dots,K\}$, we utilize the mapping $\phi$ (Eq.~\ref{eq:phi}) to project the renormalized token distribution onto the ordinal scale $\mathcal{Y}$ (Eq.~\ref{eq:y-vector}). We define the item-level ordinal random variable $Y_k$ with probability mass function

\begin{equation}
\label{eq:phi-inverse-mapping}
    P(Y_k = y) = P\left(t = \phi^{-1}(y) \mid x_k, t \in \mathcal{V}_{\text{val}}\right),
\end{equation}
where $x_k$ is the prompt for item $k$. This step formally converts the model's token prediction into a numeric behavioral response.
We compute the expected response for each item, forming the centroid of the joint distribution
\begin{equation}
\label{eq:centroid}
    \boldsymbol{\mu} = (\mu_1, \dots, \mu_K),
\end{equation}
where $\mu_k = \sum_{y \in \mathcal{Y}} y P(Y_k = y)$.
The centroid, $\boldsymbol{\mu}$, represents the model's expected position on each Likert item in the $K$-dimensional space. We define the maximum possible distance between any response vector and the centroid as the diagonal of the hypercube defined by the scale extrema
\begin{equation}
\label{eq:d-max}
    d_{\max} = \sqrt{K \,(y_{\max} - y_{\min})^2}.
\end{equation}

This normalizes distances to the unit interval and ensures comparability across instruments with different numbers of items. The Multivariate Consensus is defined as the expected logarithmic agreement of the joint distribution with respect to the centroid (Eq.~\ref{eq:consensus-multivariate}).

Intuitively, this measure assigns higher consensus when the probability mass is concentrated near the centroid across all items, and lower consensus when mass is dispersed or polarized. Unlike entropy-based measures, Consensus is sensitive to ordinal distances. When $K=1$, the multivariate formulation reduces to the original Consensus \citep{tastle2007consensus}, with the key distinction that we operate directly on the true PMF rather than on sampled responses.
Multivariate Dissension ($\mathrm{Dsn}$) is defined as 
the complement of Consensus, 
representing the degree of polarization or 
dispersion away from the mean.

\begin{equation}    
\label{eq:dissension}
\begin{aligned}
\mathrm{Dsn}(\mathbf{Y}) = 1 - \mathrm{Cns}(\mathbf{Y})
\end{aligned}
\end{equation}

\subsection{The Construct Layer}
\label{app:construct}

Our goal is to quantify differences in the dependent variable across
experimental conditions and models, while fully propagating
uncertainty induced by the model's generative process. 
We formalize inference over latent scale scores
using an analytical fixed-effects ANOVA/Hoeffding decomposition.
For a specific condition vector $\boldsymbol{\lambda}$, 
the composite scale score, $S_{\boldsymbol{\lambda}}$, 
is the aggregation of item-level random variable responses

\begin{equation}
\label{eq:s-lambda}
S_{\boldsymbol{\lambda}}
=
\sum_{k=1}^K Y_{k,\boldsymbol{\lambda}}.
\end{equation}

Let $P_{Y_{k, \boldsymbol{\lambda}}}$ denote the PMF of item $k$ under condition $\boldsymbol{\lambda}$. Given independence (Sec.~\ref{sec:independence-assumption}), the distribution of the composite score is obtained via discrete convolution
\begin{equation}
\label{eq:item-wise-convolution}
    P_{S_{\boldsymbol{\lambda}}} = P_{Y_{1, \boldsymbol{\lambda}}} \circledast P_{Y_{2, \boldsymbol{\lambda}}} \circledast \dots \circledast P_{Y_{K, \boldsymbol{\lambda}}}.
\end{equation}

\subsubsection{Factorial Decomposition Over Distributions}

We formalize an analytical ANOVA/Hoeffding decomposition of the composite scale score. 
Our framework constructs a unique family of centered effect distributions $\{E_U\}_{U \subseteq \mathcal{C}}$, representing main effects and interactions, such that the total expected score satisfies the additive identity
\begin{equation}
\label{eq:decomposition-preview}
\small
\mathbb{E}[S_{\boldsymbol{\lambda}}] = \underbrace{\mathbb{E}[S_0]}_{\text{Grand Mean}} + \sum_{c \in \mathcal{C}} \underbrace{\mathbb{E}[E_c(\lambda_c)]}_{\text{Main Effects}} + \sum_{\substack{U \subseteq \mathcal{C} \\ |U| \ge 2}} \underbrace{\mathbb{E}[E_U(\boldsymbol{\lambda}_U)]}_{\text{Interactions}},
\end{equation}

while retaining the full probability mass function for every term in the summation.

For each experimental condition  
\(\boldsymbol{\lambda} \in \mathcal{D}\), the composite scale score
$S_{\boldsymbol{\lambda}}$ 
is a discrete random variable with probability mass function
\(P_{S_{\boldsymbol{\lambda}}}\) (Eq.~\ref{eq:s-lambda}).
Rather than reducing \(S_{\boldsymbol{\lambda}}\) to a point estimate,
we treat the mapping
\begin{equation}
\boldsymbol{\lambda}
\;\mapsto\;
P_{S_{\boldsymbol{\lambda}}}    
\end{equation}

as the primitive object of analysis.  
This allows all uncertainty induced by the model's generative process to be propagated exactly through subsequent analyses.
Given the fully crossed factorial design, 
all factor combinations are symmetric.
Therefore, 
the design space $\mathcal{D}$ takes on a uniform probability
distribution $\pi$, 
assigning equal weight
to each experimental condition.

\begin{assumption}[Fully Crossed Fixed-Effects Design]
\label{ass:fully-crossed}
The experimental design space
\(
\mathcal{D} = \prod_{c \in \mathcal{C}} \{1,\dots,\Lambda_c\}
\)
is finite and fully crossed. All expectations over
\(\mathcal{D}\) are taken with respect to the uniform measure.
\end{assumption}

\begin{proposition}[Linearity of Expectation under Mixture Marginalization]
\label{prop:linearity-mixture}
Let $\mathcal{D}$ be a finite factorial design with probability measure $\pi$, and let $P_{\bar S} = \mathbb{E}_{\boldsymbol{\lambda} \sim \pi} [P_{S_{\boldsymbol{\lambda}}}]$ be the mixture distribution. Then,
\begin{equation}
\label{eq:mixture-expectation}
\mathbb{E}[\bar S]
=
\mathbb{E}_{\boldsymbol{\lambda} \sim \pi}
\!\left[
\mathbb{E}[S_{\boldsymbol{\lambda}}]
\right].
\end{equation}
\end{proposition}

The proof follows immediately from the linearity of finite sums and the definition of expectation.

\begin{corollary}[Expectation of Factorial Marginal Distributions]
\label{cor:anova-consistency}
Let $S_U(\boldsymbol{\lambda}_U)$ denote a factorial
marginal distribution defined by marginalizing $P_{S_{\boldsymbol{\lambda}}}$ over
$\boldsymbol{\lambda}_{-U}$ under the uniform measure.
Then
\[
\mathbb{E}\!\left[S_U(\boldsymbol{\lambda}_U)\right]
=
\mathbb{E}_{\boldsymbol{\lambda}_{-U}}
\!\left[
\mathbb{E}[S_{\boldsymbol{\lambda}}]
\right].
\]
\end{corollary}

The proof follows directly from Prop.~\ref{prop:linearity-mixture}
by taking $\pi$ to be the uniform distribution over
$\boldsymbol{\lambda}_{-U}$.

\subsubsection{Grand Mixture Distribution}

To establish a baseline for factorial effects, we construct the distributional equivalent of the classical ANOVA Grand Mean. We compute this as the Grand Mixture Distribution ($P_{S_0}$), obtained by averaging the composite score distributions uniformly across the entire design space $\mathcal{D}$
\begin{equation}
\label{eq:grand-mixture}
P_{S_0}
=
\mathbb{E}_{\boldsymbol{\lambda} \sim \pi}
\!\left[
P_{S_{\boldsymbol{\lambda}}}
\right]
=
\frac{1}{|\mathcal{D}|}
\sum_{\boldsymbol{\lambda} \in \mathcal{D}}
P_{S_{\boldsymbol{\lambda}}}.  
\end{equation}

The random variable \(S_0\) encapsulates the global background distribution of the construct, encompassing all sources of variation in the experiment.

\subsubsection{Marginal Distributions}
\label{app:marginal-slices}

ANOVA Marginal Mean averages responses over complementary factors to isolate a specific level. We extend it by computing the Marginal Slice Distribution ($P_{M_c(\lambda_c)}$). This is obtained by taking the uniform mixture of composite score distributions over all factors other than $c$,

{\small
\begin{equation}
\label{eq:marginal-means}
\small
\begin{split}
P_{M_c(\lambda_c)}&=\mathbb{E}_{\boldsymbol{\lambda}_{-c}}\!\left[P_{S_{(\lambda_c,\boldsymbol{\lambda}_{-c})}}\right]\\
                  &=\frac{1}{|\mathcal{D}_{-c}|}\sum_{\boldsymbol{\lambda}_{-c}\in\mathcal{D}_{-c}}P_{S_{(\lambda_c,\boldsymbol{\lambda}_{-c})}}.
\end{split}
\end{equation}
}

The random variable \(M_c(\lambda_c)\) is the predictive distribution of the
composite score when factor \(c\) is fixed at level \(\lambda_c\), averaging
uniformly over all remaining factors. 
At the expectation level (Prop.~\ref{prop:linearity-mixture}),
\(\mathbb{E}[M_c(\lambda_c)]\) equals the classical marginal mean.

\begin{corollary}[Grand mixture as an average of marginal slices]
\label{cor:grand-as-average-of-marginals}
Under Assumption~\ref{ass:fully-crossed}, for any factor \(c\),
\begin{equation}
\label{eq:grand-as-average-of-marginals}
P_{S_0}
=
\frac{1}{\Lambda_c}
\sum_{\lambda_c=1}^{\Lambda_c}
P_{M_c(\lambda_c)}.
\end{equation}
\end{corollary}

\begin{proof}
Fix \(c \in \mathcal{C}\). The grand mixture distribution is
\begin{equation}
\label{eq:proof-grand-def}
P_{S_0}
=
\frac{1}{|\mathcal{D}|}
\sum_{\boldsymbol{\lambda}\in\mathcal{D}}
P_{S_{\boldsymbol{\lambda}}}.
\end{equation}
Under Assumption~\ref{ass:fully-crossed}, the design space factors as
\(
\mathcal{D} \;=\; \{1,\dots,\Lambda_c\}\times \mathcal{D}_{-c},
\)
where
\(
\mathcal{D}_{-c} := \prod_{d\in\mathcal{C}\setminus\{c\}}\{1,\dots,\Lambda_d\}.
\)
Hence each \(\boldsymbol{\lambda}\in\mathcal{D}\) can be written uniquely as
\(\boldsymbol{\lambda}=(\lambda_c,\boldsymbol{\lambda}_{-c})\) with
\(\lambda_c\in\{1,\dots,\Lambda_c\}\) and \(\boldsymbol{\lambda}_{-c}\in\mathcal{D}_{-c}\).
We may rewrite the sum in Eq.~\eqref{eq:proof-grand-def} as a double sum
\begin{equation}
\label{eq:proof-double-sum}
P_{S_0}
=
\frac{1}{|\mathcal{D}|}
\sum_{\lambda_c=1}^{\Lambda_c}
\sum_{\boldsymbol{\lambda}_{-c}\in\mathcal{D}_{-c}}
P_{S_{(\lambda_c,\boldsymbol{\lambda}_{-c})}}.
\end{equation}
Because the design is fully crossed,
\(
|\mathcal{D}| = \Lambda_c \, |\mathcal{D}_{-c}|.
\)
Substituting this into Eq.~\eqref{eq:proof-double-sum} gives

{\small
\begin{align}
P_{S_0}
&=
\frac{1}{\Lambda_c\,|\mathcal{D}_{-c}|}
\sum_{\lambda_c=1}^{\Lambda_c}
\sum_{\boldsymbol{\lambda}_{-c}\in\mathcal{D}_{-c}}
P_{S_{(\lambda_c,\boldsymbol{\lambda}_{-c})}}
\label{eq:proof-factor-cardinality}\\[4pt]
&=
\frac{1}{\Lambda_c}
\sum_{\lambda_c=1}^{\Lambda_c}
\left(
\frac{1}{|\mathcal{D}_{-c}|}
\sum_{\boldsymbol{\lambda}_{-c}\in\mathcal{D}_{-c}}
P_{S_{(\lambda_c,\boldsymbol{\lambda}_{-c})}}
\right).
\label{eq:proof-recognize-marginal}
\end{align}
}

The expression in parentheses in Eq.~\eqref{eq:proof-recognize-marginal} is
exactly the marginal slice distribution \(P_{M_c(\lambda_c)}\) from
Eq.~\eqref{eq:marginal-means}. Therefore,
\[
P_{S_0}
=
\frac{1}{\Lambda_c}
\sum_{\lambda_c=1}^{\Lambda_c}
P_{M_c(\lambda_c)},
\]
which proves Eq.~\eqref{eq:grand-as-average-of-marginals}.
\end{proof}

\subsubsection{Expectation-Level Decomposition}
\label{app:expectation-decomposition}

Our primary estimands in the Construct Layer are predictive
distributions (PMFs) and contrast distribution constructed via
mixtures and paired subtraction. 
Because such contrasts depend on a chosen coupling, they do not, in general, admit an additive decomposition at the level of random variables.
However, our distributional framework connects to classical fixed-effects ANOVA at the \emph{expectation} level (Eq.~\ref{eq:decomposition-preview}). 
Let $\mu(\boldsymbol{\lambda})$ be the deterministic mean surface over the factorial design
\begin{equation}
\label{eq:mu-def}
\mu(\boldsymbol{\lambda})
:=
\mathbb{E}\!\left[S_{\boldsymbol{\lambda}}\right],
\qquad
\boldsymbol{\lambda}\in\mathcal{D}.
\end{equation}
Here, we state the Hoeffding/ANOVA decomposition of $\mu$ on the finite
product space $\mathcal{D}$. 
We provide the unique collection of expectation-level
effect components $\{\mu_U\}_{U\subseteq\mathcal{C}}$ that our distributional
effect estimands target in expectation.

\begin{theorem}[Expectation-Level ANOVA/Hoeffding Decomposition]
\label{thm:anova-expectation}
Assume a finite fully crossed design $\mathcal{D}$ endowed with the uniform
product measure (Assumption~\ref{ass:fully-crossed}). 
There exists a unique collection of functions
$\{\mu_U\}_{U\subseteq\mathcal{C}}$ such that
\begin{equation}
\label{eq:anova-decomp}
\mu(\boldsymbol{\lambda})
=
\sum_{U\subseteq\mathcal{C}}
\mu_U(\boldsymbol{\lambda}_U),
\end{equation}
and each component satisfies the identifiability (sum-to-zero) constraints

{\small
\begin{equation}
\label{eq:anova-identifiability}
\sum_{\lambda_c=1}^{\Lambda_c}
\mu_U(\boldsymbol{\lambda}_U)
=
0
\qquad
\text{for all }U\subseteq\mathcal{C}\text{ and all }c\in U.
\end{equation}
}
\end{theorem}

\begin{proof}
    \label{prf:anova-expectation}
We work on the finite product space $\mathcal{D}$ with the uniform product
measure implied by Assumption~\ref{ass:fully-crossed}.

For any function $f:\mathcal{D}\to\mathbb{R}$ and any subset
$U\subseteq\mathcal{C}$, define the uniform marginalization operator

{\small
\begin{align}
\label{eq:TU-def}
(T_U f)(\boldsymbol{\lambda}_U)
&:=
\mathbb{E}_{\boldsymbol{\lambda}_{-U}}\!\left[
f(\boldsymbol{\lambda}_U,\boldsymbol{\lambda}_{-U})
\right]\\
&=
\frac{1}{|\mathcal{D}_{-U}|}
\sum_{\boldsymbol{\lambda}_{-U}\in\mathcal{D}_{-U}}
f(\boldsymbol{\lambda}_U,\boldsymbol{\lambda}_{-U}),
\end{align}
}

where $\mathcal{D}_{-U}:=\prod_{c\in\mathcal{C}\setminus U}\{1,\dots,\Lambda_c\}$.
By construction, $T_U f$ depends only on $\boldsymbol{\lambda}_U$, and
$T_{\mathcal{C}} f=f$.

Define components recursively in increasing order of $|U|$
\begin{align}
\label{eq:hoeffding-recursion}
\mu_\emptyset
&:= (T_\emptyset \mu)
\qquad\text{(a constant)},\\
\mu_U(\boldsymbol{\lambda}_U)
&:=
(T_U \mu)(\boldsymbol{\lambda}_U)
-
\sum_{V\subsetneq U}
\mu_V(\boldsymbol{\lambda}_V),\\ 
&U\neq\emptyset.
\notag
\end{align}
This is well-defined because the right-hand side for $\mu_U$ only involves
$\mu_V$ with strictly smaller index sets $V\subsetneq U$, which have already
been defined.

From Eq.~\eqref{eq:hoeffding-recursion} we immediately obtain, for every
$U\subseteq\mathcal{C}$,
\begin{equation}
\label{eq:TU-sum-identity}
(T_U \mu)(\boldsymbol{\lambda}_U)
=
\sum_{V\subseteq U}\mu_V(\boldsymbol{\lambda}_V).
\end{equation}
For $U=\emptyset$ this holds by definition. 
For $U\neq\emptyset$ it is a simple
rearrangement of \eqref{eq:hoeffding-recursion}.
Setting $U=\mathcal{C}$ and using $T_{\mathcal{C}}\mu=\mu$ yields
\[
\mu(\boldsymbol{\lambda})
=
\sum_{V\subseteq\mathcal{C}}\mu_V(\boldsymbol{\lambda}_V),
\]
which is Eq.~\eqref{eq:anova-decomp}.

\paragraph{Identifiability.}
Fix $U\subseteq\mathcal{C}$ and $c\in U$. Uniformly averaging $\mu_U$ over
$\lambda_c$ gives

{\small
\begin{equation}
\mathbb{E}_{\lambda_c}\!\left[\mu_U(\boldsymbol{\lambda}_U)\right]
=
\mathbb{E}_{\lambda_c}\!\left[(T_U\mu)(\boldsymbol{\lambda}_U)\right]
-
\sum_{V\subsetneq U}\mathbb{E}_{\lambda_c}\!\left[\mu_V(\boldsymbol{\lambda}_V)\right].  
\end{equation}
}

Because the measure is a product measure, averaging $T_U\mu$ over $\lambda_c$
removes coordinate $c$ from the conditioning,
\[
\mathbb{E}_{\lambda_c}\!\left[(T_U\mu)(\boldsymbol{\lambda}_U)\right]
=
(T_{U\setminus\{c\}}\mu)(\boldsymbol{\lambda}_{U\setminus\{c\}}).
\]
For $V\subsetneq U$, if $c\in V$ then $\mathbb{E}_{\lambda_c}[\mu_V]=0$ by
induction on $|V|$. If $c\notin V$ then $\mu_V$ does not depend on $\lambda_c$
and the average leaves it unchanged. Hence,

{\small
\begin{equation}
\mathbb{E}_{\lambda_c}\!\left[\mu_U(\boldsymbol{\lambda}_U)\right]
= 
(T_{U\setminus\{c\}}\mu)(\boldsymbol{\lambda}_{U\setminus\{c\}})
-
\sum_{V\subseteq U\setminus\{c\}}\mu_V(\boldsymbol{\lambda}_V).  
\end{equation}
}

The right-hand side is zero by applying Eq.~\eqref{eq:TU-sum-identity} with
$U\setminus\{c\}$. 
Therefore
$\mathbb{E}_{\lambda_c}[\mu_U(\boldsymbol{\lambda}_U)]=0$, which under uniformity
is equivalent to Eq.~\eqref{eq:anova-identifiability}.

\paragraph{Uniqueness.}

Suppose $\{\tilde\mu_U\}$ is another collection satisfying
Eq.~\eqref{eq:anova-decomp} and Eq.~\eqref{eq:anova-identifiability}. Define
$h_U:=\tilde\mu_U-\mu_U$. Then $\sum_{U}h_U(\boldsymbol{\lambda}_U)=0$ for all
$\boldsymbol{\lambda}$ and each $h_U$ satisfies the same sum-to-zero constraints.
Apply $T_U$ to $\sum_W h_W=0$, where $W$ is any subset of $\mathcal{C}$. If $W\subseteq U$, then $T_U h_W=h_W$. If
$W\not\subseteq U$, pick $c\in W\setminus U$. 
Marginalizing over $-U$ averages
over $\lambda_c$, and the sum-to-zero constraint implies $T_U h_W=0$. Thus,
\[
0
=
T_U\!\left[\sum_{W\subseteq\mathcal{C}} h_W\right]
=
\sum_{W\subseteq U} h_W
\qquad\text{for all }U\subseteq\mathcal{C}.
\]
Taking $U=\emptyset$ gives $h_\emptyset=0$. Proceed by induction on $|U|$.
Assuming $h_V=0$ for all $V\subsetneq U$, the identity above implies $h_U=0$.
Hence $h_U\equiv 0$ for all $U$, proving uniqueness.
\end{proof}

\begin{corollary}[M\"obius inversion form]
\label{cor:mobius-form}
The components in Theorem~\ref{thm:anova-expectation} admit the explicit
inclusion--exclusion form

{\small
\begin{equation}
\label{eq:mobius-inversion}
\mu_U(\boldsymbol{\lambda}_U)
=
\sum_{V\subseteq U}
(-1)^{|U|-|V|}
\,(T_V\mu)(\boldsymbol{\lambda}_V),
\qquad U\subseteq\mathcal{C}.
\end{equation}
}

\end{corollary}

\begin{proof}
Equation~\eqref{eq:TU-sum-identity} states that for each fixed $U$,
$(T_U\mu)=\sum_{V\subseteq U}\mu_V$. This is exactly the zeta-transform relation
on the Boolean lattice $(2^{\mathcal{C}},\subseteq)$. Applying M\"obius inversion
(\citealt{rota1964foundations}) yields Eq.~\eqref{eq:mobius-inversion}.
\end{proof}

\begin{remark}
By Prop.~\ref{prop:linearity-mixture} and the definition of the global marginal
slice distribution $M_U(\boldsymbol{\lambda}_U)$ (Eq.~\eqref{eq:marginal-slice-U}),
\[
\mathbb{E}\!\left[M_U(\boldsymbol{\lambda}_U)\right]
=
(T_U\mu)(\boldsymbol{\lambda}_U).
\]
Thus, $T_U\mu$ can be interpreted as the classical marginal mean surface, and
$\mu_U$ in Eq.~\eqref{eq:mobius-inversion} are the associated ANOVA/Hoeffding
effect components targeted by our distributional effect estimands in
expectation.
\end{remark}

\subsubsection{Main Effects Distribution}
\label{app:main-effect-dists}

Classical ANOVA main effects quantify centered deviations of a
factor-level marginal mean from the grand mean. In our distributional setting,
we seek an effect distribution whose expectation recovers the ANOVA main
effect (Eq.~\ref{eq:decomposition-preview}). 
There are various choices for computing the difference distribution. 
One option is to compute the
convolution-based subtraction
$M_c(\lambda_c)-S_0$.
The main problem with this derivation is that it treats the slice distribution of the factor, $M_c(\lambda_c)$, and the grand mixture, $S_0$, as independent. 
This is an incorrect assumption because $S_0$ and $M_c(\lambda_c)$ are clearly
correlated ($M_c(\lambda_c)$ is a component of the mixture $S_0$).

Moreover, convolution of two random variables adds their variance together.
This inflates the variance associated with the main effect of factor $c$.
Specifically, the
convolution-based subtraction
$M_c(\lambda_c)-S_0$ corresponds to a \emph{between-subjects} contrast (difference
of two independent draws), which can yield substantial variability even when
the two distributions are \emph{identical}.

We therefore derive the main-effect distributions by minimizing the variance of the paired difference between the specific level and the baseline. 
This approach isolates the systematic shift attributable to the factor while removing the "noise" common to both distributions.
To achieve this, we employ a comonotone coupling (i.e., 1D monotone optimal transport), which maximizes the covariance between paired draws \citep{villani2008optimal}. 

The pairing has two components. We have a blocked construction that respects
the factorial design by holding the complementary factors
$\boldsymbol{\lambda}_{-c}$ fixed.
Within each block, we have a canonical
comonotone coupling, which minimizes the expected squared difference between paired
draws among all couplings with the same marginals.

\begin{definition}[Paired difference of two univariate discrete distributions]
\label{def:paired-difference}
Let $X$ and $Y$ be univariate discrete random variables with finite support and
CDFs $F_X$ and $F_Y$. Let $U\sim\mathrm{Unif}(0,1)$ and define the (left-continuous)
generalized inverse CDF
\[
F^{-1}(u) := \inf\{z\in\mathbb{R}: F(z)\ge u\}.
\]
The comonotone coupling is given by
$X^\uparrow := F_X^{-1}(U)$ and $Y^\uparrow := F_Y^{-1}(U)$.
We define the \emph{paired difference} random variable as
\begin{equation}
\label{eq:paired-diff-rv}
X \ominus Y \;:=\; X^\uparrow - Y^\uparrow.
\end{equation}
This construction preserves marginals ($X^\uparrow\stackrel{d}{=}X$,
$Y^\uparrow\stackrel{d}{=}Y$), hence
$\mathbb{E}[X\ominus Y]=\mathbb{E}[X]-\mathbb{E}[Y]$.
\end{definition}

Let $\mathrm{supp}(X)=\{x_1<\dots<x_m\}$ with masses $p_i=P(X=x_i)$ and define
$P_i=\sum_{r=1}^i p_r$ with $P_0=0$. Similarly, let
$\mathrm{supp}(Y)=\{y_1<\dots<y_n\}$ with masses $q_j$ and
$Q_j=\sum_{s=1}^j q_s$ with $Q_0=0$.
The comonotone coupling assigns joint mass

{\small
\begin{align}
\label{eq:comonotone-coupling-mass}
w_{ij}
&=
\max\!\Bigl\{
0,\;
\min(P_i,Q_j) - \max(P_{i-1},Q_{j-1})
\Bigr\},\\
&i=1,\dots,m,\; j=1,\dots,n.
\end{align}
}

The paired-difference PMF is the pushforward under $(x,y)\mapsto x-y$,
\begin{equation}
\label{eq:paired-diff-pmf}
P_{X\ominus Y}(z)
=
\sum_{i=1}^m\sum_{j=1}^n
w_{ij}\,\mathbbm{1}\{z=x_i-y_j\}.
\end{equation}
All terms are finite and are computed deterministically from the two input PMFs.

To respect the factorial dependence structure (i.e., $S_0$ is constructed
from the same family $\{S_{\boldsymbol{\lambda}}\}$ as $M_c(\lambda_c)$), we form
a grand-mean baseline within each block $\boldsymbol{\lambda}_{-c}$ by
averaging only over the levels of factor $c$
\begin{equation}
\label{eq:blocked-grand-mean}
P_{S_{0\mid \boldsymbol{\lambda}_{-c}}}
:=
\frac{1}{\Lambda_c}
\sum_{\lambda_c'=1}^{\Lambda_c}
P_{S_{(\lambda_c',\boldsymbol{\lambda}_{-c})}}.
\end{equation}
Mixing these blocked baselines uniformly over $\boldsymbol{\lambda}_{-c}$ recovers
the original grand mixture distribution
\(
\frac{1}{|\mathcal{D}_{-c}|}\sum_{\boldsymbol{\lambda}_{-c}} P_{S_{0\mid \boldsymbol{\lambda}_{-c}}}
= P_{S_0}.
\)

The (centered) main-effect distribution for factor $c$ at level
$\lambda_c$ is the uniform mixture of \emph{paired} within-block contrasts
\begin{equation}
\label{eq:main-effect-dist}
E_c(\lambda_c)
:=
\mathbb{E}_{\boldsymbol{\lambda}_{-c}}
\!\left[
S_{(\lambda_c,\boldsymbol{\lambda}_{-c})}
\;\ominus\;
S_{0\mid \boldsymbol{\lambda}_{-c}}
\right].
\end{equation}
At the PMF level this becomes

{\small
\begin{equation}
\label{eq:main-effect-pmf}
P_{E_c(\lambda_c)}
=
\frac{1}{|\mathcal{D}_{-c}|}
\sum_{\boldsymbol{\lambda}_{-c}\in\mathcal{D}_{-c}}
P_{\,S_{(\lambda_c,\boldsymbol{\lambda}_{-c})}\ominus S_{0\mid \boldsymbol{\lambda}_{-c}}}.
\end{equation}
}

Each paired-difference PMF in Eq.~\eqref{eq:main-effect-pmf} is computed exactly
via the mass-matching formula in Eq.~\eqref{eq:paired-diff-pmf}. 
No estimation or approximation is used at any stage.

\begin{corollary}[Expectation-level centering and sum-to-zero]
\label{cor:main-effect-anova}
For each factor $c\in\mathcal{C}$ and level $\lambda_c$,
\begin{equation}
\label{eq:main-effect-expectation}
\mathbb{E}\!\left[E_c(\lambda_c)\right]
=
\mathbb{E}\!\left[M_c(\lambda_c)\right]
-
\mathbb{E}\!\left[S_0\right].
\end{equation}
Moreover, the paired main effects satisfy the standard ANOVA sum-to-zero
constraint in expectation
\begin{equation}
\label{eq:main-effect-sum-to-zero}
\sum_{\lambda_c=1}^{\Lambda_c}\mathbb{E}\!\left[E_c(\lambda_c)\right] = 0.
\end{equation}
\end{corollary}

\begin{proof}
    \label{prf:main-effect-anova}
By definition (Eq.~\eqref{eq:main-effect-pmf}), the main-effect random variable
$E_c(\lambda_c)$ is the uniform mixture over blocks
$\boldsymbol{\lambda}_{-c}\in\mathcal{D}_{-c}$ of the paired contrast
$S_{(\lambda_c,\boldsymbol{\lambda}_{-c})}\ominus S_{0\mid \boldsymbol{\lambda}_{-c}}$.
Equivalently, we can represent the mixture hierarchically.
Draw $\boldsymbol{\lambda}_{-c}\sim \mathrm{Unif}(\mathcal{D}_{-c})$, then draw
\[
\left(
S_{(\lambda_c,\boldsymbol{\lambda}_{-c})}^\uparrow,\;
S_{0\mid \boldsymbol{\lambda}_{-c}}^\uparrow
\right)
\]
from the comonotone coupling.

Set
$E_c(\lambda_c) = S_{(\lambda_c,\boldsymbol{\lambda}_{-c})}^\uparrow
- S_{0\mid \boldsymbol{\lambda}_{-c}}^\uparrow$.
Because $\mathcal{D}_{-c}$ is finite, we may apply iterated expectations.

{\small
\begin{equation}
\label{eq:proof-iterated-exp}
\mathbb{E}\!\left[E_c(\lambda_c)\right]
=
\frac{1}{|\mathcal{D}_{-c}|}
\sum_{\boldsymbol{\lambda}_{-c}\in\mathcal{D}_{-c}}
\mathbb{E}\!\left[
S_{(\lambda_c,\boldsymbol{\lambda}_{-c})}\ominus S_{0\mid \boldsymbol{\lambda}_{-c}}
\right].
\end{equation}
}

Fix a block $\boldsymbol{\lambda}_{-c}$. By Def.~\ref{def:paired-difference},
\[
S_{(\lambda_c,\boldsymbol{\lambda}_{-c})}\ominus S_{0\mid \boldsymbol{\lambda}_{-c}}
=
S_{(\lambda_c,\boldsymbol{\lambda}_{-c})}^\uparrow
-
S_{0\mid \boldsymbol{\lambda}_{-c}}^\uparrow,
\]
where
$S_{(\lambda_c,\boldsymbol{\lambda}_{-c})}^\uparrow \stackrel{d}{=}
S_{(\lambda_c,\boldsymbol{\lambda}_{-c})}$ and
$S_{0\mid \boldsymbol{\lambda}_{-c}}^\uparrow \stackrel{d}{=}
S_{0\mid \boldsymbol{\lambda}_{-c}}$.
Therefore the expectations match their marginals

{\small
\begin{align}  
  \mathbb{E}\!\left[S_{(\lambda_c,\boldsymbol{\lambda}_{-c})}^\uparrow\right]
  &=
  \mathbb{E}\!\left[S_{(\lambda_c,\boldsymbol{\lambda}_{-c})}\right], \\
  \mathbb{E}\!\left[S_{0\mid \boldsymbol{\lambda}_{-c}}^\uparrow\right]
  &=
  \mathbb{E}\!\left[S_{0\mid \boldsymbol{\lambda}_{-c}}\right].
\end{align}
}

Using linearity of expectation, which does \emph{not} require independence,

{\small
\begin{equation}
\label{eq:proof-paired-mean}
\mathbb{E}\!\left[
S_{(\lambda_c,\boldsymbol{\lambda}_{-c})}\ominus S_{0\mid \boldsymbol{\lambda}_{-c}}
\right]
=
\mathbb{E}\!\left[S_{(\lambda_c,\boldsymbol{\lambda}_{-c})}\right]
-
\mathbb{E}\!\left[S_{0\mid \boldsymbol{\lambda}_{-c}}\right].
\end{equation}
}

Eq.~\eqref{eq:proof-paired-mean} holds for \emph{any} coupling with
the given marginals. The comonotone coupling is chosen to minimize dispersion,
but preserves the mean.

Substitute Eq.~\eqref{eq:proof-paired-mean} into Eq.~\eqref{eq:proof-iterated-exp}:

{\small
\begin{align}
\mathbb{E}\!\left[E_c(\lambda_c)\right]
&=
\frac{1}{|\mathcal{D}_{-c}|}
\sum_{\boldsymbol{\lambda}_{-c}\in\mathcal{D}_{-c}}
\left(
\mathbb{E}\!\left[S_{(\lambda_c,\boldsymbol{\lambda}_{-c})}\right]
-
\mathbb{E}\!\left[S_{0\mid \boldsymbol{\lambda}_{-c}}\right]
\right) \notag\\[4pt]
&=
\underbrace{
\frac{1}{|\mathcal{D}_{-c}|}
\sum_{\boldsymbol{\lambda}_{-c}\in\mathcal{D}_{-c}}
\mathbb{E}\!\left[S_{(\lambda_c,\boldsymbol{\lambda}_{-c})}\right]
}_{(\star)} \\
&-
\underbrace{
\frac{1}{|\mathcal{D}_{-c}|}
\sum_{\boldsymbol{\lambda}_{-c}\in\mathcal{D}_{-c}}
\mathbb{E}\!\left[S_{0\mid \boldsymbol{\lambda}_{-c}}\right]
}_{(\dagger)}.
\label{eq:proof-split}
\end{align}
}

For $(\star)$, by the definition of the marginal slice distribution
$P_{M_c(\lambda_c)}$ (Eq.~\eqref{eq:marginal-means}) and Prop.~\ref{prop:linearity-mixture},
\[
\mathbb{E}[M_c(\lambda_c)]
=
\frac{1}{|\mathcal{D}_{-c}|}
\sum_{\boldsymbol{\lambda}_{-c}\in\mathcal{D}_{-c}}
\mathbb{E}\!\left[S_{(\lambda_c,\boldsymbol{\lambda}_{-c})}\right].
\]

For $(\dagger)$, expand the blocked grand mean
(Eq.~\eqref{eq:blocked-grand-mean}) and apply Theorem~\ref{prop:linearity-mixture},

{\small
\begin{align}
\frac{1}{|\mathcal{D}_{-c}|}\sum_{\boldsymbol{\lambda}_{-c}}
\mathbb{E}\!\left[S_{0\mid \boldsymbol{\lambda}_{-c}}\right]
&=
\frac{1}{|\mathcal{D}_{-c}|}\sum_{\boldsymbol{\lambda}_{-c}}
\mathbb{E}\!\left[
\frac{1}{\Lambda_c}\sum_{\lambda_c'=1}^{\Lambda_c}
S_{(\lambda_c',\boldsymbol{\lambda}_{-c})}
\right] \\
&=
\frac{1}{|\mathcal{D}_{-c}|}\sum_{\boldsymbol{\lambda}_{-c}}
\frac{1}{\Lambda_c}\sum_{\lambda_c'=1}^{\Lambda_c}
\mathbb{E}\!\left[S_{(\lambda_c',\boldsymbol{\lambda}_{-c})}\right] \\
&=
\frac{1}{|\mathcal{D}|}\sum_{\boldsymbol{\lambda}\in\mathcal{D}}
\mathbb{E}\!\left[S_{\boldsymbol{\lambda}}\right]
=
\mathbb{E}[S_0].
\end{align}
}

We used $|\mathcal{D}|=\Lambda_c\,|\mathcal{D}_{-c}|$ and the definition of
the grand mixture $S_0$ (Eq.~\eqref{eq:grand-mixture}).

Substituting these two identifications into Eq.~\eqref{eq:proof-split} yields
\[
\mathbb{E}\!\left[E_c(\lambda_c)\right]
=
\mathbb{E}\!\left[M_c(\lambda_c)\right]
-
\mathbb{E}\!\left[S_0\right],
\]
which proves Eq.~\eqref{eq:main-effect-expectation}.

Sum Eq.~\eqref{eq:proof-iterated-exp} over $\lambda_c\in\{1,\dots,\Lambda_c\}$

{\small
\begin{align}
\sum_{\lambda_c=1}^{\Lambda_c} &\mathbb{E}\!\left[E_c(\lambda_c)\right]
=
\frac{1}{|\mathcal{D}_{-c}|}
\sum_{\boldsymbol{\lambda}_{-c}\in\mathcal{D}_{-c}}
\sum_{\lambda_c=1}^{\Lambda_c} \\
&\mathbb{E}\!\left[
S_{(\lambda_c,\boldsymbol{\lambda}_{-c})}\ominus S_{0\mid \boldsymbol{\lambda}_{-c}}
\right].
\end{align}
}

Apply Eq.~\eqref{eq:proof-paired-mean} inside the inner sum

{\small
\begin{align}
&\sum_{\lambda_c=1}^{\Lambda_c}
\mathbb{E}\!\left[
S_{(\lambda_c,\boldsymbol{\lambda}_{-c})}\ominus S_{0\mid \boldsymbol{\lambda}_{-c}}
\right]
= \\
&\sum_{\lambda_c=1}^{\Lambda_c}
\left(
\mathbb{E}[S_{(\lambda_c,\boldsymbol{\lambda}_{-c})}] 
-
\mathbb{E}[S_{0\mid \boldsymbol{\lambda}_{-c}}]
\right)
= \\
&\sum_{\lambda_c=1}^{\Lambda_c}\mathbb{E}[S_{(\lambda_c,\boldsymbol{\lambda}_{-c})}]
-
\Lambda_c \,\mathbb{E}[S_{0\mid \boldsymbol{\lambda}_{-c}}].
\end{align}
}

But by definition of the blocked grand mean and linearity of expectation,
\[
\mathbb{E}[S_{0\mid \boldsymbol{\lambda}_{-c}}]
=
\frac{1}{\Lambda_c}\sum_{\lambda_c'=1}^{\Lambda_c}
\mathbb{E}[S_{(\lambda_c',\boldsymbol{\lambda}_{-c})}],
\]
so
$\Lambda_c \mathbb{E}[S_{0\mid \boldsymbol{\lambda}_{-c}}]
=
\sum_{\lambda_c'=1}^{\Lambda_c}\mathbb{E}[S_{(\lambda_c',\boldsymbol{\lambda}_{-c})}]$.
Hence the inner expression is identically zero for every
$\boldsymbol{\lambda}_{-c}$, and therefore the outer average is also zero
\[
\sum_{\lambda_c=1}^{\Lambda_c}\mathbb{E}\!\left[E_c(\lambda_c)\right] = 0,
\]
which proves Eq.~\eqref{eq:main-effect-sum-to-zero}.
\end{proof}

\subsubsection{Interaction Effects Distribution}
\label{app:interactions}

Whereas main effects quantify additive shifts attributable to individual factors,
interaction effects quantify departures from additivity when multiple factors
are fixed jointly. In classical fixed-effects ANOVA/Hoeffding decompositions,
the $U$-way interaction component is obtained by M\"obius inversion of marginal
means on the Boolean lattice of factor subsets.

In our framework, since we are dealing directly with PMFs of the population,
we define interaction effects using a
paired, design-blocked contrast.
We hold the complementary factors $\boldsymbol{\lambda}_{-U}$ fixed (blocking).
Within each block, we couple all terms in the alternating sum via a
shared latent quantile variable (comonotone coupling). 
This yields a paired distribution for the interaction contrast.

For any nonempty subset of factors $U \subseteq \mathcal{C}$ and level vector
$\boldsymbol{\lambda}_U$, define the $U$-way marginal slice distribution as the
uniform mixture over all remaining factors
\begin{equation}
\label{eq:marginal-slice-U}
P_{M_U(\boldsymbol{\lambda}_U)}
=
\frac{1}{|\mathcal{D}_{-U}|}
\sum_{\boldsymbol{\lambda}_{-U}\in\mathcal{D}_{-U}}
P_{S_{(\boldsymbol{\lambda}_U,\boldsymbol{\lambda}_{-U})}}.
\end{equation}
We further set $M_\emptyset := S_0$.

Fix a nonempty subset of factors $U\subseteq\mathcal{C}$ and a realization of
the complementary factors $\boldsymbol{\lambda}_{-U}\in\mathcal{D}_{-U}$.
For any $V\subseteq U$, define the within-block design space
\[
\mathcal{D}_{U\setminus V}
:=
\prod_{c\in U\setminus V}\{1,\dots,\Lambda_c\}.
\]
We derive the blocked marginal slice distribution as the uniform mixture
over the remaining factors in $U\setminus V$, holding $\boldsymbol{\lambda}_{-U}$
fixed

\begin{align}
\label{eq:blocked-slice}
&P_{M_{V\mid \boldsymbol{\lambda}_{-U}}(\boldsymbol{\lambda}_V)} \\
&:=\frac{1}{|\mathcal{D}_{U\setminus V}|}
\sum_{\boldsymbol{\lambda}_{U\setminus V}\in\mathcal{D}_{U\setminus V}}
P_{S_{(\boldsymbol{\lambda}_V,\boldsymbol{\lambda}_{U\setminus V},\boldsymbol{\lambda}_{-U})}}.
\end{align}

This unifies several cases,
\begin{align}
&M_{U\mid \boldsymbol{\lambda}_{-U}}(\boldsymbol{\lambda}_U)
\equiv S_{(\boldsymbol{\lambda}_U,\boldsymbol{\lambda}_{-U})}, \\
&M_{\emptyset\mid \boldsymbol{\lambda}_{-U}}
=
\frac{1}{|\mathcal{D}_U|}
\sum_{\boldsymbol{\lambda}_U\in\mathcal{D}_U}
S_{(\boldsymbol{\lambda}_U,\boldsymbol{\lambda}_{-U})},
\end{align}

where $M_{\emptyset\mid \boldsymbol{\lambda}_{-U}}$ is the blocked grand mean
(over factors $U$) within the fixed complement $\boldsymbol{\lambda}_{-U}$.

Averaging blocked slices over $\boldsymbol{\lambda}_{-U}$ recovers the global
marginal slices from Eq.~\eqref{eq:marginal-slice-U}.
\begin{equation}
\label{eq:blocked-to-global-slice}
P_{M_V(\boldsymbol{\lambda}_V)}
=
\frac{1}{|\mathcal{D}_{-U}|}
\sum_{\boldsymbol{\lambda}_{-U}\in\mathcal{D}_{-U}}
P_{M_{V\mid \boldsymbol{\lambda}_{-U}}(\boldsymbol{\lambda}_V)}.
\end{equation}
In particular, taking $V=\emptyset$ yields
$\frac{1}{|\mathcal{D}_{-U}|}\sum_{\boldsymbol{\lambda}_{-U}}P_{M_{\emptyset\mid \boldsymbol{\lambda}_{-U}}}
= P_{S_0}$.

To define a distribution for an alternating sum without imposing independence,
we use a shared-quantile (comonotone) coupling.

\begin{definition}[Comonotone-coupled linear contrast]
\label{def:comonotone-contrast}
Let $\{X_r\}_{r=1}^R$ be univariate discrete random variables with finite
support and CDFs $\{F_r\}_{r=1}^R$. Let $U^\ast\sim\mathrm{Unif}(0,1)$ and define
$X_r^\uparrow := F_r^{-1}(U^\ast)$ using the (left-continuous) generalized
inverse from Def.~\ref{def:paired-difference}. For coefficients
$\boldsymbol{\alpha}=(\alpha_1,\dots,\alpha_R)\in\mathbb{R}^R$, define the
comonotone-coupled linear contrast
\begin{equation}
\label{eq:comonotone-linear-contrast}
\mathcal{L}^\uparrow(\boldsymbol{\alpha};X_1,\dots,X_R)
:=
\sum_{r=1}^R \alpha_r X_r^\uparrow.
\end{equation}
For $R=2$ and $\boldsymbol{\alpha}=(1,-1)$, this reduces to the paired
difference $X_1\ominus X_2$ (Def.~\ref{def:paired-difference}).
\end{definition}

Because each $X_r^\uparrow=F_r^{-1}(U^\ast)$ is a step function of $U^\ast$,
the contrast in Eq.~\eqref{eq:comonotone-linear-contrast} is piecewise constant
on a finite partition of $[0,1]$.
Let
\[
\mathcal{B}
:=
\{0,1\}\;\cup\;
\bigcup_{r=1}^R
\left\{
F_r(x)\,:\, x\in\mathrm{supp}(X_r)
\right\},
\]
and let $0=u_0<u_1<\dots<u_L=1$ be the sorted distinct elements of $\mathcal{B}$.
For $\ell=1,\dots,L$, define
\[
z_\ell
:=
\sum_{r=1}^R \alpha_r F_r^{-1}(u_\ell).
\]
Then the PMF of $Z:=\mathcal{L}^\uparrow(\boldsymbol{\alpha};X_1,\dots,X_R)$ is
\begin{equation}
\label{eq:comonotone-contrast-pmf}
P_Z(z)
=
\sum_{\ell=1}^L
(u_\ell-u_{\ell-1})\,
\mathbbm{1}\{z=z_\ell\},
\end{equation}
aggregating masses for identical $z_\ell$ values. This computation is exact and
deterministic.

Fix $U\subseteq\mathcal{C}$ with $|U|\ge 2$. For a given level vector
$\boldsymbol{\lambda}_U$ and block $\boldsymbol{\lambda}_{-U}$, define the
within-block interaction-effect random variable as the paired M\"obius contrast
over blocked slices
\begin{equation}
\label{eq:interaction-block}
H_{U\mid \boldsymbol{\lambda}_{-U}}(\boldsymbol{\lambda}_U)
:=
\sum_{V\subseteq U}
(-1)^{|U|-|V|}
\left(M_{V\mid \boldsymbol{\lambda}_{-U}}(\boldsymbol{\lambda}_V)\right)^\uparrow,
\end{equation}
where all terms are coupled comonotonically via the same $U^\ast$ in
Def.~\ref{def:comonotone-contrast}. The PMF of
$H_{U\mid \boldsymbol{\lambda}_{-U}}(\boldsymbol{\lambda}_U)$ is computed exactly
using Eq.~\eqref{eq:comonotone-contrast-pmf} with $R=2^{|U|}$ and coefficients
$\alpha_V=(-1)^{|U|-|V|}$.

Finally, we define the global $U$-way interaction effect distribution by mixing
uniformly over complementary factors:
\begin{equation}
\label{eq:interaction-global}
P_{E_U(\boldsymbol{\lambda}_U)}
=
\frac{1}{|\mathcal{D}_{-U}|}
\sum_{\boldsymbol{\lambda}_{-U}\in\mathcal{D}_{-U}}
P_{H_{U\mid \boldsymbol{\lambda}_{-U}}(\boldsymbol{\lambda}_U)}.
\end{equation}

\paragraph{Two-way interactions.}
For $U=\{a,b\}$, within a fixed block $\boldsymbol{\lambda}_{-ab}$ the paired
two-way interaction is

\begin{align}
H_{\{a,b\}\mid \boldsymbol{\lambda}_{-ab}}(\lambda_a,\lambda_b)
&=
S_{(\lambda_a,\lambda_b,\boldsymbol{\lambda}_{-ab})}^\uparrow \\
&-
M_{a\mid \boldsymbol{\lambda}_{-ab}}^\uparrow(\lambda_a) \\
&-
M_{b\mid \boldsymbol{\lambda}_{-ab}}^\uparrow(\lambda_b) \\
&+
M_{\emptyset\mid \boldsymbol{\lambda}_{-ab}}^\uparrow, 
\end{align}
and $P_{E_{\{a,b\}}(\lambda_a,\lambda_b)}$ is the uniform mixture of these
within-block PMFs over $\boldsymbol{\lambda}_{-ab}$.

\begin{corollary}[Expectation-level ANOVA consistency and centering]
\label{cor:paired-interaction-anova}
Let $\mu(\boldsymbol{\lambda})=\mathbb{E}[S_{\boldsymbol{\lambda}}]$ and let
$\{\mu_U\}_{U\subseteq\mathcal{C}}$ be the unique Hoeffding/ANOVA components from
Theorem~\ref{thm:anova-expectation}. For any $U\subseteq\mathcal{C}$ with
$|U|\ge 2$ and any $\boldsymbol{\lambda}_U$,
\begin{equation}
\label{eq:paired-interaction-anova}
\mathbb{E}\!\left[E_U(\boldsymbol{\lambda}_U)\right]
=
\mu_U(\boldsymbol{\lambda}_U).
\end{equation}
In particular, the interaction effects are centered in expectation. For any
$c\in U$,
\begin{equation}
\label{eq:paired-interaction-sum-to-zero}
\sum_{\lambda_c=1}^{\Lambda_c}
\mathbb{E}\!\left[E_U(\boldsymbol{\lambda}_U)\right]
=
0.
\end{equation}
\end{corollary}

\begin{proof}
Fix $\boldsymbol{\lambda}_{-U}$. By construction, the comonotone coupling in
Eq.~\eqref{eq:interaction-block} preserves the marginal distribution of each
term $M_{V\mid \boldsymbol{\lambda}_{-U}}(\boldsymbol{\lambda}_V)$, hence it does
not alter expectations. Therefore, by linearity of expectation,

{\small
\begin{equation}
\label{eq:proof-interaction-block-mean}
\mathbb{E}\!\left[H_{U\mid \boldsymbol{\lambda}_{-U}}(\boldsymbol{\lambda}_U)\right]
=
\sum_{V\subseteq U}(-1)^{|U|-|V|}
\mathbb{E}\!\left[M_{V\mid \boldsymbol{\lambda}_{-U}}(\boldsymbol{\lambda}_V)\right].
\end{equation}
}

Next, expand the blocked slice definition (Eq.~\eqref{eq:blocked-slice}) and
apply linearity of expectation to obtain

\begin{align}
&\mathbb{E}\!\left[M_{V\mid \boldsymbol{\lambda}_{-U}}(\boldsymbol{\lambda}_V)\right] \\
&=
\frac{1}{|\mathcal{D}_{U\setminus V}|}
\sum_{\boldsymbol{\lambda}_{U\setminus V}\in\mathcal{D}_{U\setminus V}}
\mu(\boldsymbol{\lambda}_V,\boldsymbol{\lambda}_{U\setminus V},\boldsymbol{\lambda}_{-U}).  
\end{align}

Now average Eq.~\eqref{eq:proof-interaction-block-mean} uniformly over
$\boldsymbol{\lambda}_{-U}$ (Eq.~\eqref{eq:interaction-global}). Because the
design is a finite product space, averaging over $\boldsymbol{\lambda}_{-U}$
and over $\boldsymbol{\lambda}_{U\setminus V}$ is equivalent to averaging over
all complementary coordinates $\boldsymbol{\lambda}_{-V}$. Thus

\begin{align}
&\mathbb{E}\!\left[E_U(\boldsymbol{\lambda}_U)\right] \\
&=
\frac{1}{|\mathcal{D}_{-U}|}
\sum_{\boldsymbol{\lambda}_{-U}}
\mathbb{E}\!\left[H_{U\mid \boldsymbol{\lambda}_{-U}}(\boldsymbol{\lambda}_U)\right]\\
&=
\sum_{V\subseteq U}(-1)^{|U|-|V|}
\underbrace{
\frac{1}{|\mathcal{D}_{-V}|}
\sum_{\boldsymbol{\lambda}_{-V}}
\mu(\boldsymbol{\lambda}_V,\boldsymbol{\lambda}_{-V})
}_{=(T_V\mu)(\boldsymbol{\lambda}_V)}.
\end{align}

The final expression is exactly the M\"obius inversion / Hoeffding component
$\mu_U(\boldsymbol{\lambda}_U)$ \citep{rota1964foundations}, equivalent to the recursion in
Theorem~\ref{thm:anova-expectation}.
This proves Eq.~\eqref{eq:paired-interaction-anova}. The centering property
Eq.~\eqref{eq:paired-interaction-sum-to-zero} then follows from the
identifiability constraints in Theorem~\ref{thm:anova-expectation}.
\end{proof}

\subsubsection{Hypothesis Testing}
\label{sec:hypothesis-testing}

We distinguish our hypothesis evaluation from null-hypothesis significance
testing (NHST). Rather than accept--reject decisions based on $p$-values, which
presuppose sampling variability from finite data, we evaluate directional
scientific claims using exact predictive effect distributions computed
analytically from the model's next-token probabilities.

In our framework, the primitive objects of inference are centered effect random
variables, namely paired main effects $E_c(\lambda_c)$
(Eq.~\eqref{eq:main-effect-dist}) and paired interaction effects
$E_U(\boldsymbol{\lambda}_U)$ (Eq.~\eqref{eq:interaction-global}). Their PMFs are
computed deterministically via finite mixtures and paired comonotone contrast
operations.
Accordingly, there is no sampling error. Uncertainty in these effect
distributions reflects intrinsic stochasticity of the model's predictive
behavior under the stimulus (i.e., the aleatoric uncertainty).

We summarize each centered effect distribution $E$ using
mean shift ($\mathbb{E}[E]$),
predictive dispersion ($\operatorname{SD}(E)$),
discrete Probability of Direction ($\mathrm{dPD}(E)$),
and a standardized signal-to-noise ratio ($\mathrm{SNR}(E)$) based on the SD
of the paired effect distribution.

\subsubsection{Discrete Probability of Direction}
\label{app:dpd}

To evaluate directional hypotheses, we adopt a discrete formulation of the
Probability of Direction (PD). In Bayesian inference, PD is typically defined
for continuous posteriors, where $\mathbb{P}(E=0)=0$ \citep{Makowski2019}. Here,
predictive effect distributions are discrete and may place non-zero mass on
exact equality. We therefore treat $\{E=0\}$ as a distinct behavioral outcome.

Let $E$ denote a centered effect random variable, either $E_c(\lambda_c)$ or
$E_U(\boldsymbol{\lambda}_U)$. The discrete Probability of Direction (dPD) is
\begin{equation}
\label{eq:dpd}
\mathrm{dPD}(E)
=
\max\!\bigl(
\mathbb{P}(E > 0),
\mathbb{P}(E < 0)
\bigr),
\end{equation}
dPD quantifies the extent to which the effect exhibits consistent
directionality, irrespective of sign.

\subsubsection{Effect Size (Mean Shift)}

We define effect size as the expectation of the centered effect distribution
\begin{equation}
\label{eq:effect-size}
\mathbb{E}[E]
=
\sum_{z} z \, P_{E}(z),
\end{equation}
where the sum ranges over the finite support of $E$.
$\mathbb{E}[E]$ is the magnitude of the effect $E$.

\subsubsection{Predictive Dispersion (SD)}

Intrinsic variability of the centered effect is captured by its predictive
standard deviation
\begin{equation}
\label{eq:sd-effect}
\operatorname{SD}(E)
=
\sqrt{
\sum_{z}
\bigl(z - \mathbb{E}[E]\bigr)^2
\, P_{E}(z)
}.
\end{equation}
Because our effect distributions are constructed via paired contrasts,
$\operatorname{SD}(E)$ reflects residual predictive heterogeneity in the effect
itself, rather than dispersion induced by taking differences of independent
draws.

\subsubsection{Signal-to-Noise Ratio}
\label{app:snr}

We summarize standardized effect magnitude using a signal-to-noise ratio
computed directly on the paired effect distribution
\begin{equation}
\label{eq:snr}
\mathrm{SNR}(E)
=
\frac{|\mathbb{E}[E]|}{\operatorname{SD}(E)}.
\end{equation}
When $\operatorname{SD}(E)=0$, the effect is deterministic under the paired
contrast. In this case $\mathrm{SNR}(E)$ is treated as
$+\infty$ if $\mathbb{E}[E]\neq 0$ and $0$ if $\mathbb{E}[E]=0$.

The signal-to-noise ratio \(\mathrm{SNR}\) is
closely related to standardized effect sizes such as Cohen's \(d\)
\citep{cohen2013statistical}, which quantifies mean deviations relative to
within-group variability. However, in our setting, variability reflects intrinsic
stochasticity of the language model rather than sampling error. We therefore
use conventional heuristics (e.g., \(\approx 0.2, 0.5, 0.8\) for small, medium,
large) only as descriptive reference points, not as decision thresholds.

\subsubsection{Interpretation}
\label{app:metric-interpretation}
Inferential conclusions are drawn from the joint configuration of
$\mathrm{dPD}(E)$, $\mathbb{E}[E]$, $\operatorname{SD}(E)$, and $\mathrm{SNR}(E)$.
High $\mathrm{dPD}(E)$ indicates strong directional regularity, while large
$\mathrm{SNR}(E)$ indicates that the expected shift dominates the intrinsic
predictive dispersion of the paired effect distribution. Discrepancies between
these quantities distinguish regimes such as stable but weak effects, large but
ambivalent effects, or low-signal effects with high dispersion (Table~\ref{tab:synthetic-dpd-snr}).

\begin{table*}[t]
\centering
\begin{tabular}{l l r r r r}
\toprule
Case &
$P(E)$ &
$\mathbb{E}[E]$ &
SD &
SNR &
dPD \\
\midrule

A &
$\{-25{:}0.9,\,-20{:}0.1\}$ &
$-24.5$ &
$1.5$ &
$16.3$ &
$1.00$ \\

B &
$\{1{:}0.99,\,100{:}0.01\}$ &
$11.0$ &
$99.0$ &
$0.11$ &
$1.00$ \\

C &
$\{100{:}0.51,\,-1{:}0.49\}$ &
$50.5$ &
$49.5$ &
$1.02$ &
$0.51$ \\

D &
$\{5{:}0.34,\,-5{:}0.33,\,0{:}0.33\}$ &
$0.05$ &
$4.1$ &
$0.01$ &
$0.34$ \\

\bottomrule
\end{tabular}
\caption{
Illustrative predictive effect distributions exhibiting distinct behavioral
regimes. Cases~A and~B have identical directional consistency
(\(\mathrm{dPD}=1\)) but sharply different standardized magnitudes
(\(\mathrm{SNR}\)). Case~C shows a large expected effect despite weak
directional consistency. Case~D exhibits neither directional regularity nor an
appreciable standardized signal.
}
\label{tab:synthetic-dpd-snr}
\end{table*}

\subsubsection{Trend Effects}
\label{app:trend-effects}

Some experimental factors admit an inherent ordering and a meaningful numeric
scale 
(e.g., model size in billions of parameters or a discretized
continuous manipulation intensity). 
To quantify the marginal \emph{per-unit}
sensitivity of the composite construct to such a factor without imposing a
parametric functional form, we compute paired finite
difference distributions across adjacent factor levels and aggregate them into an average
trend distribution.

Let $c\in\mathcal{C}$ denote an ordered factor with $L=\Lambda_c$ levels that
admit numeric values
\[
\lambda^{(c)}_1 < \lambda^{(c)}_2 < \dots < \lambda^{(c)}_L
\qquad (\lambda^{(c)}_\ell \in \mathbb{R}).
\]
Adjacent spacings may be uneven. Define increments
\[
\delta^{(c)}_\ell
=
\lambda^{(c)}_{\ell+1}-\lambda^{(c)}_\ell
\;>\;0,
\qquad
\ell=1,\dots,L-1.
\]

To respect the factorial design, we compare adjacent levels \emph{within each
block} of the remaining factors $\boldsymbol{\lambda}_{-c}\in\mathcal{D}_{-c}$.
For each block and each adjacent pair of ordered levels, define the paired
local difference random variable
\begin{align}
\label{eq:local-diff-block}
D^{(c)}_{\ell\mid \boldsymbol{\lambda}_{-c}}:=S_{(\lambda^{(c)}_{\ell+1},\boldsymbol{\lambda}_{-c})}
\;\ominus\;
S_{(\lambda^{(c)}_{\ell},\boldsymbol{\lambda}_{-c})}&, \\
\ell=1,\dots,L-1&,
\end{align}
where $\ominus$ denotes the paired comonotone difference from
Def.~\ref{def:paired-difference}. Each PMF
$P_{D^{(c)}_{\ell\mid \boldsymbol{\lambda}_{-c}}}$ is computed exactly via the
discrete mass-matching formula in Eq.~\eqref{eq:paired-diff-pmf}.

We then marginalize over blocks by a uniform mixture
\begin{equation}
\label{eq:local-diff}
P_{D^{(c)}_\ell}
=
\frac{1}{|\mathcal{D}_{-c}|}
\sum_{\boldsymbol{\lambda}_{-c}\in\mathcal{D}_{-c}}
P_{D^{(c)}_{\ell\mid \boldsymbol{\lambda}_{-c}}}.
\end{equation}
Intuitively, $D^{(c)}_\ell$ is the predictive distribution of the change in the
composite score induced by moving from $\lambda^{(c)}_\ell$ to
$\lambda^{(c)}_{\ell+1}$, averaging uniformly over all other experimental
factors while preserving a paired (within-block) comparison.

Define the per-unit local slope random variable as the rescaled difference
\begin{equation}
\label{eq:scaled-local-contrast}
\Delta^{(c)}_\ell
=
\frac{D^{(c)}_\ell}{\delta^{(c)}_\ell}.
\end{equation}
Equivalently, $P_{\Delta^{(c)}_\ell}$ is the pushforward of $P_{D^{(c)}_\ell}$
under the map $z\mapsto z/\delta^{(c)}_\ell$; its support is rescaled by
$1/\delta^{(c)}_\ell$ while probability masses are preserved.

To aggregate unevenly spaced local slopes into a single per-unit trend effect,
we define the \emph{per-unit trend} random variable $\gamma_c$ as the
length-weighted mixture of scaled local slopes:
\begin{align}
\label{eq:trend-mixture}
&P_{\gamma_c}
=
\sum_{\ell=1}^{L-1}
w^{(c)}_\ell
\, P_{\Delta^{(c)}_\ell}, \\
&w^{(c)}_\ell
=
\frac{\delta^{(c)}_\ell}{\sum_{m=1}^{L-1}\delta^{(c)}_m}
=
\frac{\delta^{(c)}_\ell}{\lambda^{(c)}_L-\lambda^{(c)}_1}&.
\end{align}

Dispersion or multimodality in $P_{\gamma_c}$ indicates heterogeneous local
behavior (e.g., non-monotonic or regime-dependent trends) across adjacent
factor-level transitions.

\begin{proposition}[Expectation-level secant-slope identity]
\label{prop:trend-expectation}
The expected per-unit trend equals the secant slope of the ordered-factor
marginal mean across the observed range
\begin{equation}
\label{eq:trend-expectation}
\mathbb{E}[\gamma_c]
=
\frac{
\mathbb{E}[M_c(\lambda^{(c)}_L)]
-
\mathbb{E}[M_c(\lambda^{(c)}_1)]
}{
\lambda^{(c)}_L-\lambda^{(c)}_1
}.
\end{equation}
\end{proposition}

\begin{proof}
Fix $\ell\in\{1,\dots,L-1\}$ and a block $\boldsymbol{\lambda}_{-c}$.
By Def.~\ref{def:paired-difference}, the paired difference preserves the mean
\[
\mathbb{E}\!\left[D^{(c)}_{\ell\mid \boldsymbol{\lambda}_{-c}}\right]
=
\mathbb{E}\!\left[S_{(\lambda^{(c)}_{\ell+1},\boldsymbol{\lambda}_{-c})}\right]
-
\mathbb{E}\!\left[S_{(\lambda^{(c)}_{\ell},\boldsymbol{\lambda}_{-c})}\right].
\]
Averaging uniformly over $\boldsymbol{\lambda}_{-c}$ and using linearity of
expectation gives
\[
\mathbb{E}[D^{(c)}_\ell]
=
\mathbb{E}[M_c(\lambda^{(c)}_{\ell+1})]
-
\mathbb{E}[M_c(\lambda^{(c)}_{\ell})],
\]
because $M_c(\lambda^{(c)}_\ell)$ is the uniform mixture of
$S_{(\lambda^{(c)}_\ell,\boldsymbol{\lambda}_{-c})}$ over
$\boldsymbol{\lambda}_{-c}$ (Eq.~\eqref{eq:marginal-means}), and expectations
commute with finite mixtures (Theorem~\ref{prop:linearity-mixture}).

By definition of $\Delta^{(c)}_\ell$,
\[
\mathbb{E}[\Delta^{(c)}_\ell]
=
\frac{\mathbb{E}[D^{(c)}_\ell]}{\delta^{(c)}_\ell}.
\]
Taking expectations in the mixture definition \eqref{eq:trend-mixture} yields

\begin{align}
\mathbb{E}[\gamma_c]
&=
\sum_{\ell=1}^{L-1} w^{(c)}_\ell\,\mathbb{E}[\Delta^{(c)}_\ell]\\
&=
\frac{1}{\sum_{m=1}^{L-1}\delta^{(c)}_m}
\sum_{\ell=1}^{L-1} \\
&\left(
\mathbb{E}[M_c(\lambda^{(c)}_{\ell+1})]
-
\mathbb{E}[M_c(\lambda^{(c)}_{\ell})]
\right),
\end{align}
which telescopes to
\(
(\mathbb{E}[M_c(\lambda^{(c)}_L)]-\mathbb{E}[M_c(\lambda^{(c)}_1)])/
(\lambda^{(c)}_L-\lambda^{(c)}_1)
\),
proving Eq.~\eqref{eq:trend-expectation}.
\end{proof}

\paragraph{Directional interpretation.}
The full predictive distribution $P_{\gamma_c}$ supports directional summaries
(e.g., $\mathbb{P}(\gamma_c>0)$, $\mathbb{P}(\gamma_c<0)$) and can be summarized
using the same quantities defined for other predictive contrasts (e.g., dPD,
SD, and SNR based on $\operatorname{SD}(\gamma_c)$).

\subsection{Prompt template}

Figure \ref{fig:prompt-setup} shows how we form the prompt template and cross Target Country with each model.

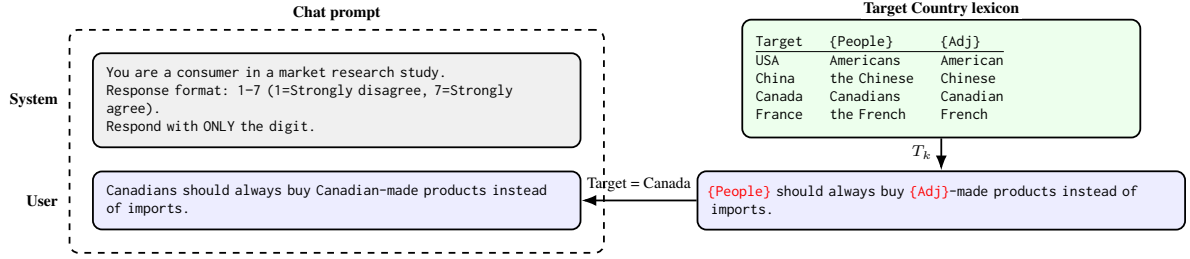
\begin{figure*}[t]
\centering
\resizebox{\textwidth}{!}{%
\begin{tikzpicture}[
  font=\scriptsize,
  msg/.style={rounded corners, draw, thick, align=left, inner sep=6pt},
  sys/.style={msg, fill=gray!12, text width=7.2cm},
  usr/.style={msg, fill=blue!7, text width=7.2cm},
  lex/.style={msg, fill=green!7, text width=5.8cm},
  temp/.style={msg, fill=green!7, text width=5.8cm},
  arrow/.style={-Latex, thick}
]

\node[sys,
  label={[font=\scriptsize\bfseries, fill=white, inner sep=15pt]west:System}] (sys) {%
\ttfamily
You are a consumer in a market research study.\\
Response format: 1--7 (1=Strongly disagree, 7=Strongly agree).\\
Respond with ONLY the digit.
};

\node[usr, below=0.35cm of sys,
  label={[font=\scriptsize\bfseries, fill=white, inner sep=15pt]west:User}] (inst) {%
\ttfamily
Canadians should always buy Canadian-made products instead of imports.
};

\node[draw, dashed, thick, rounded corners, inner sep=10pt,
      fit=(sys)(inst),
      label={[font=\scriptsize\bfseries]north:Chat prompt}] (chatprompt) {};

\node[usr, right=18mm of inst,
  label={[font=\scriptsize\bfseries, fill=white, inner sep=5pt, xshift=-3mm]north:$T_k$}] (tmpl) {%
\ttfamily
\textcolor{red}{\{People\}} should always buy
\textcolor{red}{\{Adj\}}-made products instead of imports.
};

\node[lex, above=5mm of tmpl,
  label={[font=\scriptsize\bfseries, fill=white, inner sep=1pt]north:Target Country lexicon}] (lex) {%
\ttfamily
\begin{tabular}{@{}l l l@{}}
Target & \{People\} & \{Adj\} \\ \hline
USA    & Americans   & American \\
China  & the Chinese & Chinese \\
Canada & Canadians   & Canadian \\
France & the French  & French \\
\end{tabular}
};

\draw[arrow] (lex.south) -- node[above, font=\scriptsize]{} (tmpl.north);
\draw[arrow] (tmpl.west) -- node[above, font=\scriptsize]{Target = Canada} (inst.east);

 
\end{tikzpicture}%
}
\caption{Prompt setup: a fixed system prompt enforces the response schema, while Target Country (sample shown for Canada) deterministically instantiates item templates via a lexicon. Example shown is ${T_1}$ from the CETSCALE; each CETSCALE item has its own template.}
\label{fig:prompt-setup}
\end{figure*}

\subsection{Additional Analysis}
\label{app:model-size}
\paragraph{Effect of Model Size.}

While our primary factorial design treats language models as discrete categorical subjects, model capacity is inherently continuous.
Prior work has shown that scaling language models can systematically alter social and behavioral tendencies, including the amplification of biased responses \citep{wadi2025a}.
We therefore investigate whether ethnocentric tendencies, as measured by the CETSCALE, exhibit a systematic trend as a function of model size.
Specifically, we ask whether increasing the number of parameters induces a consistent directional change in the latent ethnocentrism score.
To address this question, we conduct a trend-effects analysis (derived in App.~\ref{app:trend-effects}).

We form four within-family size series, each evaluated under the same 
\textsc{Target} manipulation and CETSCALE instrument as in the main experiment.
\textsc{Aya} (8B, 32B), \textsc{Gemma~3} (1B, 4B, 12B, 27B), \textsc{Ministral}
(3B, 8B, 14B), and \textsc{Qwen3} (0.6B, 4B, 8B, 14B, 32B, Next-80B).
\footnote{Llama~3.3 (70B) does not form a multi-point series and is excluded from
trend estimation.} 
The trend effect of model size on ethnocentrism (Table~\ref{tab:trend-effect}) is large ($\mathbb{E}=1.65$, $\mathrm{SNR}=5.87$) and directionally consistent ($\mathrm{dPD}>0.99$) for \textsc{Aya} model family. 
This indicates steep linear scaling in ethnocentric tendencies. 
For every additional billion parameters, 
the model's expected ethnocentrism score increases by 1.65 points.
Over the range evaluated (8B to 32B),
this effectively shifts a model from the lowest observed ethnocentrism to the highest, purely by scaling.
For other models, size does not have a meaningful effect (low $\mathrm{SNR}$ and $\mathrm{dPD}$) on ethnocentrism.

\subsection{Additional tables/figures}

\begin{itemize}
  \item Table~\ref{tab:cns-entropy} shows a comparison of Multivariate Consensus versus Entropy for the country-of-origin experiment.
  \item Figure~\ref{fig:human-llm} shows a comparison of the human baseline with the LLMs tested on the CETSCALE.
  \item Table~\ref{tab:coefficient-table-marginal} shows the full ANOVA decomposition (main and interaction effects) of the composite CETSCALE score.
  \item Figure~\ref{fig:country-main} shows the main effect of the Target Country on the CETSCALE.
  \item Table~\ref{tab:country-contrasts} shows the pairwise marginal contrasts of the Target Countries.
  \item Figure~\ref{fig:specific-pmf} shows the exact PMF of the LLMs on the composite CETSCALE.
  \item Table~\ref{tab:trend-effect} shows the effect of LLM size (number of parameters) on the CETSCALE.
  \item Table~\ref{tab:aggregate-vs-factorial} shows a comparison between the estimated effect using aggregation (baseline) compared with factorial decomposition (our proposed method).
\end{itemize}

\begin{table}[t]
\centering
\small
\begin{tabular}{llrr}
\toprule
Model & Country & $Cns$ & Entropy \\
\midrule
\multirow[t]{4}{*}{Aya 32B} 
 & Canada & 0.941 & 4.945 \\
 & China  & 0.939 & 5.150 \\
 & France & 0.937 & 5.010 \\
 & USA    & 0.924 & 7.315 \\
\cline{1-4}
\multirow[t]{4}{*}{Gemma 3 27B} 
 & Canada & 0.959 & 3.565 \\
 & China  & 0.941 & 4.886 \\
 & France & 0.940 & 4.944 \\
 & USA    & 0.948 & 4.533 \\
\cline{1-4}
\multirow[t]{4}{*}{Llama 3.3 70B} 
 & Canada & 0.924 & 5.620 \\
 & China  & 0.894 & 8.579 \\
 & France & 0.882 & 5.800 \\
 & USA    & 0.947 & 3.929 \\
\cline{1-4}
\multirow[t]{4}{*}{Ministral 14B} 
 & Canada & 0.669 & 35.519 \\
 & China  & 0.646 & 37.521 \\
 & France & 0.668 & 36.444 \\
 & USA    & 0.660 & 35.759 \\
\cline{1-4}
\multirow[t]{4}{*}{Qwen3 Next 80B} 
 & Canada & 0.929 & 6.861 \\
 & China  & 0.886 & 9.008 \\
 & France & 0.886 & 8.381 \\
 & USA    & 0.935 & 5.950 \\
\bottomrule
\end{tabular}
\caption{Consensus and Entropy of Models and Target Countries}
\label{tab:cns-entropy}
\end{table}

\begin{figure}[t]
  \includegraphics[width=\columnwidth]{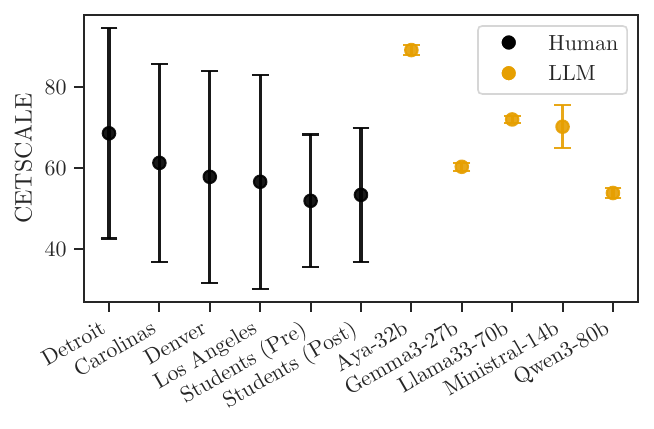}
  \caption{
Comparison of CETSCALE scores for the USA target condition between
human populations reported in \citet{shimp1987consumer} and LLMs evaluated
in this study.
Points denote expected values, and error bars indicate standard deviation.
All scores are based on the 17-item CETSCALE using a 7-point Likert scale.
Notably, several LLMs exhibit ethnocentrism levels comparable to or exceeding
those observed in highly ethnocentric human samples, while displaying
substantially lower intrinsic dispersion.
}
  \label{fig:human-llm}
\end{figure}

\begin{figure}[t]
  \includegraphics[width=\columnwidth]{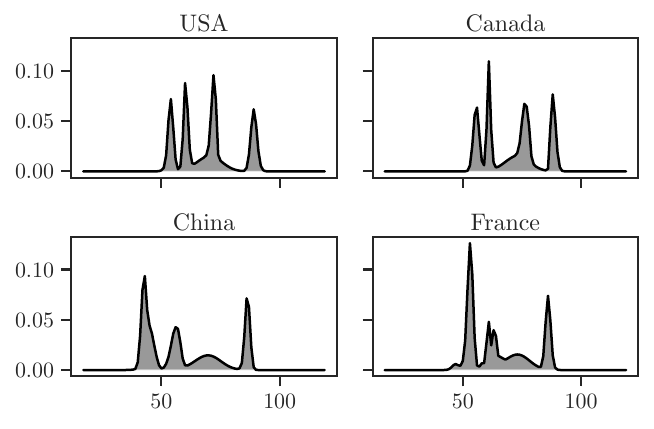}
  \caption{Main Effect of Target Country}
  \label{fig:country-main}
\end{figure}

\begin{table*}[t]
  \caption{Main effect and interaction table for the factorial model.}
  \label{tab:coefficient-table-marginal}
    \begin{center}
    \begin{small}
      \begin{sc}
\begin{tabular}{lrrll}
\toprule
\textbf{Parameter} & $\mathbb{E}$ & $\mathrm{SD}$ & $\mathrm{SNR}$ & $\mathrm{dPD}$ \\
\midrule
$\mu_{\emptyset}$ & 66.25 & 14.21 &  &  \\
\midrule
Aya-32B & 21.19 & 13.01 & 1.63 & 0.99 \\
Gemma 3-27B & -11.99 & 13.15 & 0.91 & 0.77 \\
Llama 3.3-70B & 0.31 & 12.79 & 0.02 & 0.55 \\
Ministral-14B & 5.11 & 9.59 & 0.53 & 0.62 \\
Qwen3 Next-80B & -14.63 & 12.17 & 1.20 & 0.93 \\
\midrule
USA & 2.84 & 5.43 & 0.52 & 0.55 \\
Canada & 4.62 & 5.72 & 0.81 & 0.90 \\
China & -6.46 & 6.19 & 1.04 & 0.95 \\
France & -1.00 & 4.72 & 0.21 & 0.54 \\
\midrule
Aya-32B $\times$ USA & -1.18 & 3.03 & 0.39 & 0.55 \\
Aya-32B $\times$ Canada & -3.77 & 2.88 & 1.31 & 0.94 \\
Aya-32B $\times$ China & 5.35 & 4.53 & 1.18 & 0.77 \\
Aya-32B $\times$ France & -0.39 & 2.89 & 0.14 & 0.42 \\
\midrule
Gemma 3-27B $\times$ USA & 3.21 & 5.34 & 0.60 & 0.63 \\
Gemma 3-27B $\times$ Canada & 2.14 & 5.63 & 0.38 & 0.47 \\
Gemma 3-27B $\times$ China & -5.40 & 8.63 & 0.63 & 0.61 \\
Gemma 3-27B $\times$ France & 0.05 & 5.34 & 0.01 & 0.51 \\
\midrule
Llama 3.3-70B $\times$ USA & 2.59 & 5.86 & 0.44 & 0.58 \\
Llama 3.3-70B $\times$ Canada & 5.20 & 5.49 & 0.95 & 0.77 \\
Llama 3.3-70B $\times$ China & -4.29 & 9.12 & 0.47 & 0.58 \\
Llama 3.3-70B $\times$ France & -3.50 & 5.55 & 0.63 & 0.67 \\
\midrule
Ministral-14B $\times$ USA & -4.00 & 3.10 & 1.29 & 0.88 \\
Ministral-14B $\times$ Canada & -2.99 & 3.16 & 0.95 & 0.73 \\
Ministral-14B $\times$ China & 4.54 & 4.11 & 1.11 & 0.72 \\
Ministral-14B $\times$ France & 2.45 & 2.98 & 0.82 & 0.87 \\
\midrule
Qwen3 Next-80B $\times$ USA & -0.61 & 2.83 & 0.22 & 0.58 \\
Qwen3 Next-80B $\times$ Canada & -0.57 & 2.87 & 0.20 & 0.46 \\
Qwen3 Next-80B $\times$ China & -0.20 & 5.11 & 0.04 & 0.46 \\
Qwen3 Next-80B $\times$ France & 1.39 & 2.81 & 0.49 & 0.53 \\
\bottomrule
\end{tabular}
      \end{sc}
    \end{small}
  \end{center}
  \vskip -0.1in
\end{table*}

\begin{table}[t]
  \caption{Pairwise marginal contrasts between Target Countries.}
  \label{tab:country-contrasts}
  \begin{center}
    \begin{small}
      \begin{sc}
\begin{tabular}{lrrrr}
\toprule
Contrast & $\mathbb{E}$ & $\mathrm{SD}$ & $\mathrm{SNR}$ & $\mathrm{dPD}$ \\
\midrule
USA - Canada & -1.78 & 1.83 & 0.97 & 0.74 \\
USA - China & 9.30 & 6.93 & 1.34 & 0.95 \\
USA - France & 3.84 & 4.40 & 0.87 & 0.80 \\
Canada - China & 11.08 & 7.60 & 1.46 & 1.00 \\
Canada - France & 5.62 & 5.05 & 1.11 & 0.84 \\
China - France & -5.47 & 3.79 & 1.44 & 0.80 \\
\bottomrule
\end{tabular}
      \end{sc}
    \end{small}
  \end{center}
  \vskip -0.1in
\end{table}

\begin{figure*}[t]
  \vskip 0.2in
  \begin{center}
    \begin{subfigure}{0.3\textwidth}
      \centering
      \includegraphics[width=\linewidth]{figures/specific-pmf-1.pdf}
      \caption{Aya Expanse 32B}
    \end{subfigure}
    \hfill
    \begin{subfigure}{0.3\textwidth}
      \centering
      \includegraphics[width=\linewidth]{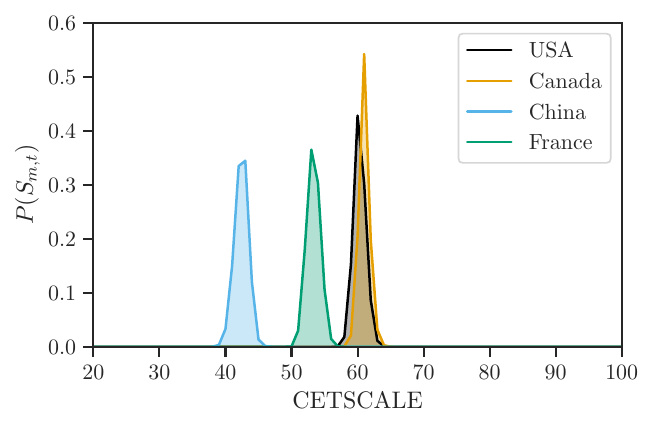}
      \caption{Gemma 3 27B}
    \end{subfigure}
    \hfill
    \begin{subfigure}{0.3\textwidth}
      \centering
      \includegraphics[width=\linewidth]{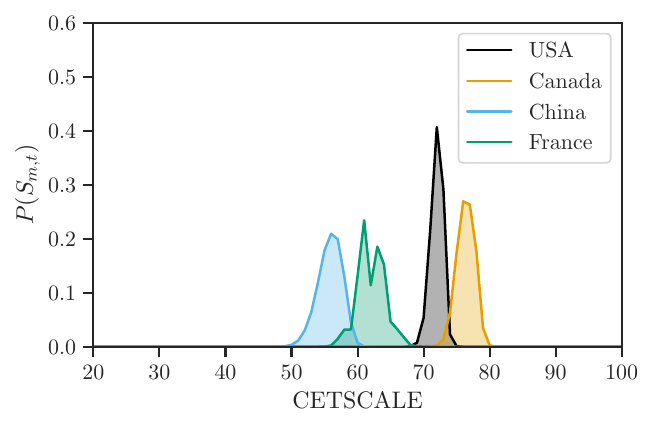}
      \caption{Llama 3.3 70B}
    \end{subfigure}

    \vspace{1em} 

    \begin{subfigure}{0.3\textwidth}
      \centering
      \includegraphics[width=\linewidth]{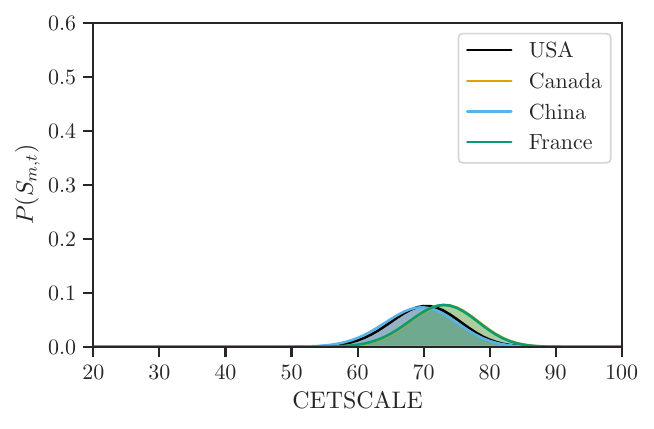}
      \caption{Ministral 14B}
    \end{subfigure}
    \hfil 
    \begin{subfigure}{0.3\textwidth}
      \centering
      \includegraphics[width=\linewidth]{figures/specific-pmf-5.pdf}
      \caption{Qwen3 Next 80B}
    \end{subfigure}

    \caption{Distribution of CETSCALE for Model and Target Country.}
    \label{fig:specific-pmf}
  \end{center}
  \vskip -0.2in
\end{figure*}

\begin{table}[t]
  \caption{Trend effects for model size (billions of parameters).}
  \label{tab:trend-effect}
  \begin{center}
    \begin{small}
      \begin{sc}
\begin{tabular}{lrrrr}
\toprule
Model Family & $\mathbb{E}$ & $\mathrm{SD}$ & $\mathrm{SNR}$ & $\mathrm{dPD}$ \\
\midrule
Aya & 1.65 & 0.28 & 5.87 & 1.00 \\
Gemma & 1.12 & 7.21 & 0.16 & 0.70 \\
Ministral & 0.87 & 1.12 & 0.78 & 0.56 \\
Qwen3 & -0.00 & 1.79 & 0.00 & 0.89 \\
\bottomrule
\end{tabular}

      \end{sc}
    \end{small}
  \end{center}
  \vskip -0.1in
\end{table}

\begin{table*}[t]
  \centering
  \caption{Aggregate versus factorial decomposition of Target-Country effects
  on CETSCALE ethnocentrism, computed with identical OT-coupled estimators.
  The aggregate benchmark (collapsing over \textsc{Model}) reports a single
  Target-Country main effect that is, by construction, uniform across models.
  The factorial \textsc{Model}\,$\times$\,\textsc{Target} interaction isolates
  the country-of-origin bias that aggregation cannot represent. Bold interactions
  indicate a sign reversal relative to the aggregate main effect (i.e.,
  conditions for which the aggregate view is directionally incorrect).}
  \label{tab:aggregate-vs-factorial}
  \begin{small}
  \begin{sc}
  \begin{tabular}{llrrrc}
  \toprule
  & & \multicolumn{4}{c}{Isolated Effect} \\
  \cmidrule(lr){3-6}
  Target & Model (Origin) & $\mathbb{E}$ & SD & SNR & dPD \\
  \midrule
  \multicolumn{6}{l}{\textit{Panel A. Aggregate benchmark (Target main effect; collapses over Model)}} \\
  USA    & --- (all models) &  2.84 & 5.43 & 0.52 & 0.55 \\
  Canada & --- (all models) &  4.62 & 5.72 & 0.81 & 0.90 \\
  China  & --- (all models) & -6.46 & 6.19 & 1.04 & 0.95 \\
  France & --- (all models) & -1.00 & 4.72 & 0.21 & 0.54 \\
  \midrule
  \multicolumn{6}{l}{\textit{Panel B. Factorial interaction (Model $\times$ Target; isolated country-of-origin bias)}} \\
  \multirow{5}{*}{USA}
    & Llama 3.3 70B (US)$^{\dagger}$ &  2.59 & 5.86 & 0.44 & 0.58 \\
    & Gemma 3 27B (US)$^{\dagger}$   &  3.21 & 5.34 & 0.60 & 0.63 \\
    & Aya Expanse 32B (CA)           & -1.18 & 3.03 & 0.39 & 0.55 \\
    & Ministral 14B (FR)             & \textbf{-4.00} & 3.10 & 1.29 & 0.88 \\
    & Qwen3 Next 80B (CN)            & -0.61 & 2.83 & 0.22 & 0.58 \\
  \cmidrule(lr){2-6}
  \multirow{5}{*}{Canada}
    & Llama 3.3 70B (US)             &  5.20 & 5.49 & 0.95 & 0.77 \\
    & Gemma 3 27B (US)              &  2.14 & 5.63 & 0.38 & 0.47 \\
    & Aya Expanse 32B (CA)$^{\dagger}$ & \textbf{-3.77} & 2.88 & 1.31 & 0.94 \\
    & Ministral 14B (FR)            & \textbf{-2.99} & 3.16 & 0.95 & 0.73 \\
    & Qwen3 Next 80B (CN)          & \textbf{-0.57} & 2.87 & 0.20 & 0.46 \\
  \cmidrule(lr){2-6}
  \multirow{5}{*}{China}
    & Llama 3.3 70B (US)           & -4.29 & 9.12 & 0.47 & 0.58 \\
    & Gemma 3 27B (US)            & -5.40 & 8.63 & 0.63 & 0.61 \\
    & Aya Expanse 32B (CA)        & \textbf{ 5.35} & 4.53 & 1.18 & 0.77 \\
    & Ministral 14B (FR)         & \textbf{ 4.54} & 4.11 & 1.11 & 0.72 \\
    & Qwen3 Next 80B (CN)$^{\dagger}$ & -0.20 & 5.11 & 0.04 & 0.46 \\
  \cmidrule(lr){2-6}
  \multirow{5}{*}{France}
    & Llama 3.3 70B (US)         & \textbf{-3.50} & 5.55 & 0.63 & 0.67 \\
    & Gemma 3 27B (US)          &  0.05 & 5.34 & 0.01 & 0.51 \\
    & Aya Expanse 32B (CA)      &  0.39 & 2.89 & 0.14 & 0.42 \\
    & Ministral 14B (FR)$^{\dagger}$ & \textbf{ 2.45} & 2.98 & 0.82 & 0.87 \\
    & Qwen3 Next 80B (CN)      &  1.39 & 2.81 & 0.49 & 0.53 \\
  \bottomrule
  \end{tabular}
  \end{sc}
  \end{small}
  \\[2pt]
  \begin{flushleft}
  \footnotesize
  \textit{Note.} All effects are isolated distributions computed via the same
  optimal-transport (comonotone) coupling used throughout the paper; $\mathbb{E}$
  is the mean shift on the composite CETSCALE score, SNR $=|\mathbb{E}|/\text{SD}$,
  and dPD is the discrete Probability of Direction. Panel A is the entire
  country signal an aggregate benchmark can express; it is identical across
  models by construction. Panel B isolates the \textsc{Model}\,$\times$\,\textsc{Target}
  interaction, which is orthogonal to both marginals and therefore
  identically zero under aggregation. Boldface marks interactions whose sign
  reverses relative to the corresponding Panel A main effect. $^{\dagger}$Ingroup
  cell (Target country matches the model's country of origin).
  \end{flushleft}
\end{table*}

\subsection{Cost of Monte Carlo sampling}
\label{app:sampling-cost}

To quantify the cost of the standard Monte Carlo sampling paradigm, 
we treat the exact PMFs as ground truth and characterize the two distinct costs a sampling-based estimator of the composite CETSCALE score would incur: 
(i) stochastic sampling noise from finite $N$, and
(ii) systematic decoding bias from non-default temperature or top-$p$. 
All quantities are reported in composite-score units (the $17$--$119$ CETSCALE scale), 
directly comparable to the effect sizes in Table~\ref{tab:main-effects}.

\subsubsection{Sampling Noise}
\label{app:sampling-noise}

Because CETSCALE items are conditionally independent
(Sec.~\ref{sec:independence-assumption}), 
a single sampled composite response has variance $\sigma^2=\sum_{k=1}^{K}\mathrm{Var}(Y_k)$, 
and the sample-mean estimator of the composite score has standard error
$\mathrm{SE}=\sigma/\sqrt{N}$. 
Table~\ref{tab:sampling-se} reports this SE across per-condition budgets $N$, alongside each model's multivariate Consensus
($\mathrm{Cns}$). 
We verified these analytical values against a Monte Carlo simulation ($500$ replications per cell), 
which matched to within-simulation error.

\begin{table*}[htbp]
  \caption{Standard error (composite-score units) of a sampling-based estimator
  of the CETSCALE score, averaged per model, as a function of the per-condition
  sample size $N$. Models are sorted by multivariate Consensus.}
  \label{tab:sampling-se}
  \centering
  \begin{small}
  \begin{sc}
  \begin{tabular}{lrrrrr}
  \toprule
  Model & $\mathrm{Cns}$ & $N{=}10$ & $N{=}10^2$ & $N{=}10^3$ & $N{=}10^6$ \\
  \midrule
  Gemma 3-27B    & 0.947 & 0.304 & 0.096 & 0.030 & 0.001 \\
  Aya-32B        & 0.935 & 0.358 & 0.113 & 0.036 & 0.001 \\
  Llama 3.3-70B  & 0.912 & 0.493 & 0.156 & 0.049 & 0.002 \\
  Qwen3 Next-80B & 0.909 & 0.517 & 0.164 & 0.052 & 0.002 \\
  Ministral-14B  & 0.661 & 1.656 & 0.524 & 0.166 & 0.005 \\
  \bottomrule
  \end{tabular}
  \end{sc}
  \end{small}
\end{table*}

Sampling error is non-trivial at realistic budgets. 
At $N{=}100$, the SE ranges from $0.10$ (Gemma) to $0.52$ (Ministral) composite-score points. 
The latter is a meaningful fraction of several effects in Table~\ref{tab:main-effects}. 
The error decays only as $1/\sqrt{N}$, and hence each tenfold increase in $N$ reduces SE by only $\approx 3.16\times$, 
and nonzero error persists even at $N{=}10^6$. 
Moreover, the magnitude of this noise is anti-correlated with Consensus
($r=-1.00$, $p<0.001$, across the $20$ $\textsc{Model}\times\textsc{Target}$ conditions). 
Low-consensus item PMFs demand far more samples for the same precision, 
so Ministral ($\mathrm{Cns}{=}0.66$) requires $\approx 5\times$ the samples of Gemma ($\mathrm{Cns}{=}0.95$) at every $N$
(Figure~\ref{fig:sampling-se}).

\begin{figure}[htbp]
  \centering
  \includegraphics[width=\columnwidth]{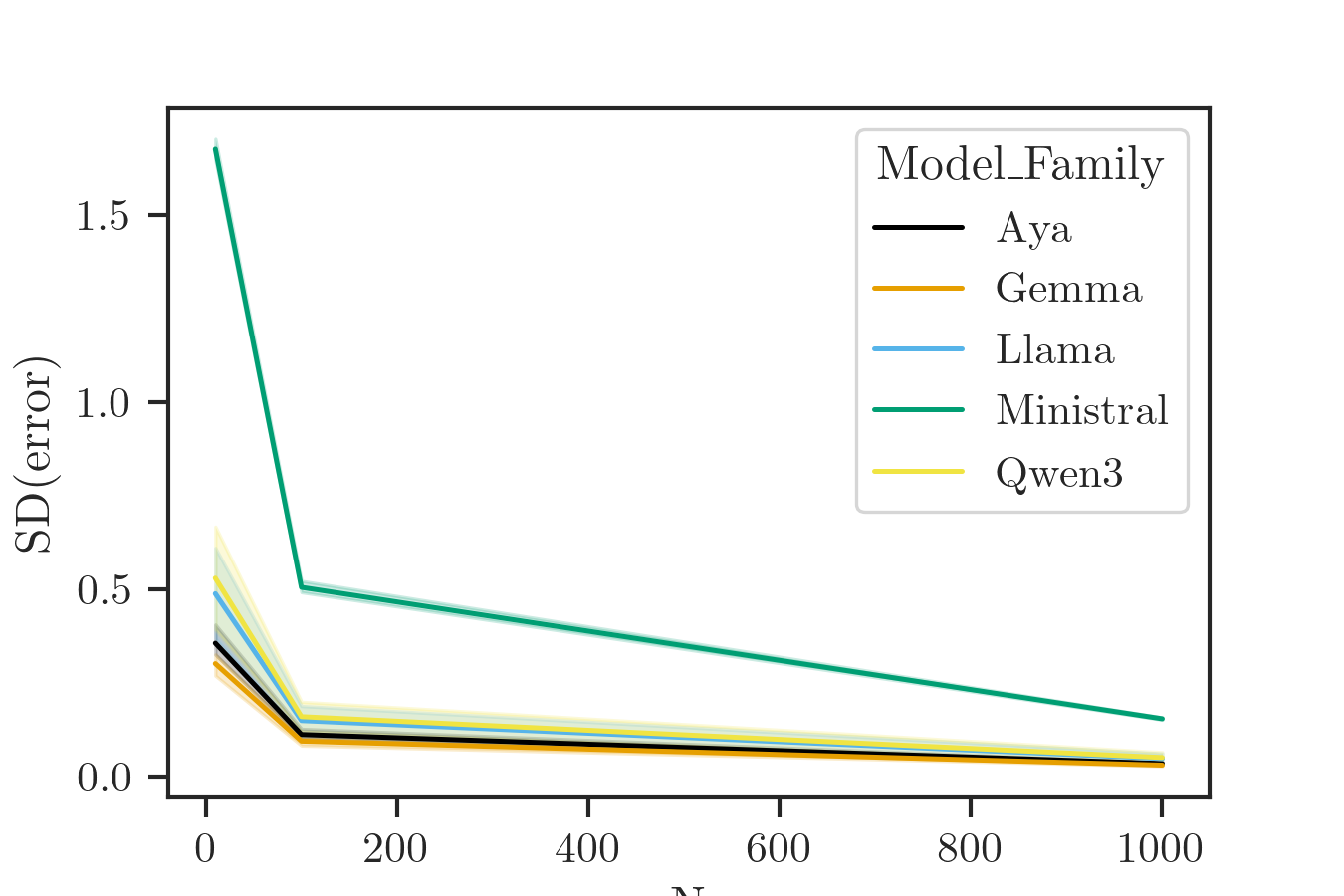}
  \caption{Standard error of the sampling-based CETSCALE estimator versus
  per-condition sample size $N$, by model. Low-consensus models (e.g.,
  Ministral) incur substantially larger error at every budget.}
  \label{fig:sampling-se}
\end{figure}

\subsubsection{Decoding Bias}
\label{app:decoding-bias}

Whereas sampling noise vanishes as $N\to\infty$, 
non-default decoding introduces a systematic bias that no amount of sampling can correct, 
because temperature and top-$p$ reshape the target distribution itself. 
We quantify this deterministically. 
For each condition we recompute the exact expected composite score under a grid of temperatures 
$T\in\{0.1,0.5,1.0,1.3\}$ and
$\text{top-}p\in\{0.8,0.9,1.0\}$, 
and measure the bias relative to the model's
native ($T{=}1.0$, $\text{top-}p{=}1.0$) distribution
(Table~\ref{tab:decoding-bias}).

\begin{table*}[htbp]
  \caption{Decoding-induced bias (composite-score units) relative to the native
  distribution, aggregated over the $T\times\text{top-}p$ grid. $M_{|Bias|}$ is
  the mean absolute bias; $SD_{Bias}$ its standard deviation across the grid.
  Models sorted by Consensus.}
  \label{tab:decoding-bias}
  \centering
  \begin{small}
  \begin{sc}
  \begin{tabular}{lrrr}
  \toprule
  Model & $\mathrm{Cns}$ & $M_{|Bias|}$ & $SD_{Bias}$ \\
  \midrule
  Gemma 3-27B    & 0.95 & 0.18 & 0.25 \\
  Aya-32B        & 0.94 & 0.24 & 0.38 \\
  Llama 3.3-70B  & 0.91 & 0.59 & 0.75 \\
  Qwen3 Next-80B & 0.91 & 0.33 & 0.53 \\
  Ministral-14B  & 0.66 & 1.38 & 1.39 \\
  \bottomrule
  \end{tabular}
  \end{sc}
  \end{small}
\end{table*}

The distortion is small for high-consensus models (Gemma: $0.18$ points) but substantial for low-consensus Ministral 
($M_{|Bias|}=1.38$, reaching up to $+2.79$ points at low temperature and reversing sign at $T{=}1.3$;
Figures~\ref{fig:bias-temp}--\ref{fig:bias-topp}), comparable to the effects in
Table~\ref{tab:main-effects}. 
As with sampling noise, the bias is strongly
anti-correlated with Consensus ($r=-0.85$, $p<0.001$, $n{=}20$ conditions). 
Intuitively, the results confirm that a sharp, high-consensus PMF has almost no dispersed mass for temperature or nucleus truncation to reshape, 
whereas a flat, low-consensus PMF has its mean substantially shifted.

\begin{figure}[htbp]
  \centering
  \includegraphics[width=\columnwidth]{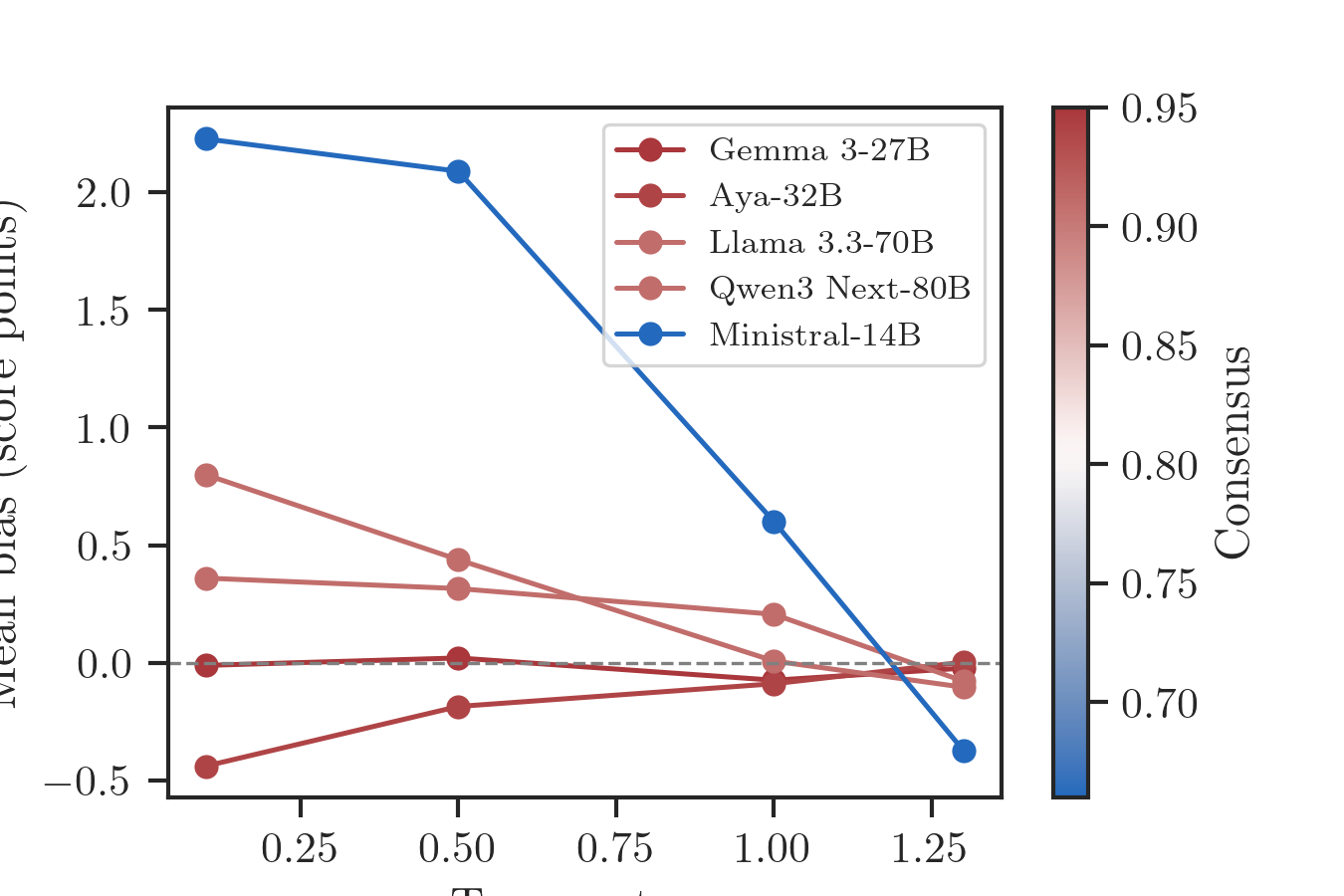}
  \caption{Mean decoding bias versus temperature, colored by Consensus.
  Low-consensus models exhibit large, sign-changing distortions.}
  \label{fig:bias-temp}
\end{figure}

\begin{figure}[htbp]
  \centering
  \includegraphics[width=\columnwidth]{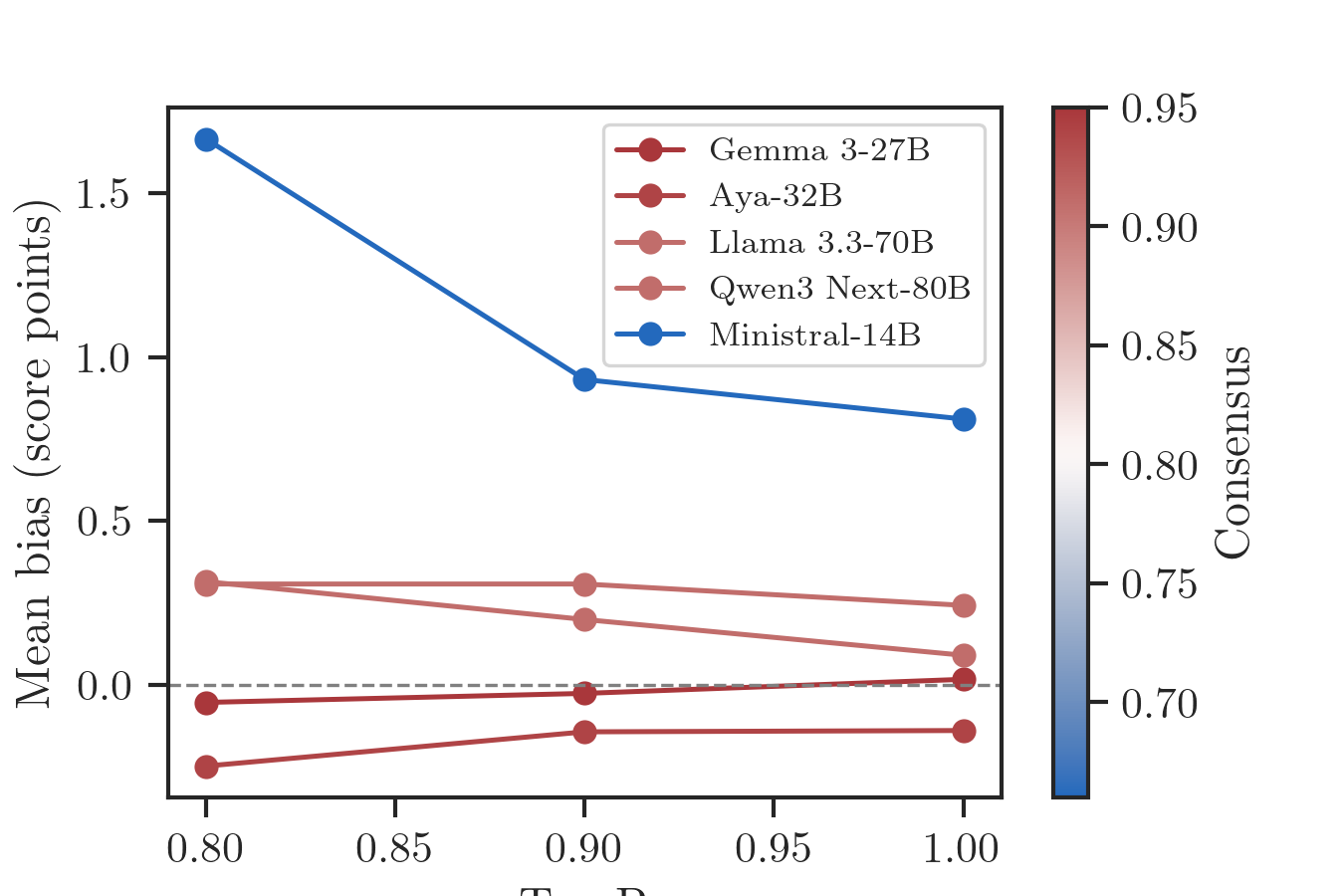}
  \caption{Mean decoding bias versus top-$p$, colored by Consensus.}
  \label{fig:bias-topp}
\end{figure}

\subsection{Reliability of sampling-based estimators under ideal decoding}
\label{app:flip-analysis}

The exact-PMF framework eliminates sampling error by construction. 
To quantify what a researcher stands to lose by instead adopting the standard Monte Carlo
sampling paradigm, we ask a decision-relevant question: 
\emph{what fraction of the causal effects we report could a sampling-based study
get wrong (i.e., recover with the opposite sign) purely as an artifact of finite
sampling?}

\paragraph{Setup.}
We treat our exact effect distributions as ground truth. 
For each factorial parameter (i.e., the \textsc{Model} and \textsc{Target} main effects and every
$\textsc{Model}\times\textsc{Target}$ interaction) we compute its exact effect size $E_{\text{exact}}$ 
using the sum-to-zero contrasts consistent with
Theorem~\ref{thm:anova-consistency-main}. 
We then model a researcher who estimates the same contrast from $N$ sampled CETSCALE responses per condition.
Because items are conditionally independent (Sec.~\ref{sec:independence-assumption}),
a single composite draw has variance $\sum_{k=1}^{K}\mathrm{Var}(Y_{k})$, 
so the contrast estimator is asymptotically

$$\hat{E}\sim\mathcal{N}\!\left(E_{\text{exact}},\ \sum_{c}\alpha_c^2\,
\sigma_c^2/N\right)$$,

where $\alpha_c$ are the contrast coefficients 
and $\sigma_c^2$ is the composite variance of condition $c$. 
The \emph{flip probability} ($P_{\text{flip}}$) is the probability that the sampling estimate carries the wrong sign,

\begin{equation}
\label{eq:flip-prob}
P_{\text{flip}}
=
\Phi\!\left(
-\,\frac{\mathrm{sign}(E_{\text{exact}})\,\mathbb{E}[\hat{E}]}
{\mathrm{SD}(\hat{E})}
\right),
\end{equation}

where $\Phi$ is the standard normal CDF. 
We evaluate the realistic per-condition budgets $N\in\{10,100,1000\}$ under the default decoding setting
($T{=}1.0$, top-$p{=}1.0$), isolating the cost of sampling noise alone.

\paragraph{Results.}
Table~\ref{tab:flip-by-effectsize} bins parameters by their exact absolute effect
size (in composite-score units on the $17$--$119$ CETSCALE) and reports the mean flip probability. 
Figure~\ref{fig:p-flip} shows the full relationship between
effect size and $P_{\text{flip}}$.

\begin{table*}[htbp]
  \caption{Mean sign-flip probability of a sampling-based estimator under default
  decoding ($T{=}1.0$, top-$p{=}1.0$), stratified by exact effect size (composite-score
  units) and per-condition sample size $N$. \textit{Params} is the share of all
  factorial parameters falling in each bin.}
  \label{tab:flip-by-effectsize}
  \centering
  \begin{small}
  \begin{sc}
  \begin{tabular}{lrrrr}
  \toprule
  $|E_{\text{exact}}|$ & Params (\%) & $N{=}10$ & $N{=}100$ & $N{=}1000$ \\
  \midrule
  $[0,1)$   & 24.1 & 0.180 & 0.063 & 0.013 \\
  $[1,2)$   &  6.9 & 0.004 & 0.000 & 0.000 \\
  $[2,5)$   & 41.4 & 0.002 & 0.000 & 0.000 \\
  $[5,10)$  & 17.2 & 0.000 & 0.000 & 0.000 \\
  $\ge 10$  & 10.3 & 0.000 & 0.000 & 0.000 \\
  \bottomrule
  \end{tabular}
  \end{sc}
  \end{small}
\end{table*}

\begin{figure}[htbp]
  \centering
  \includegraphics[width=\columnwidth]{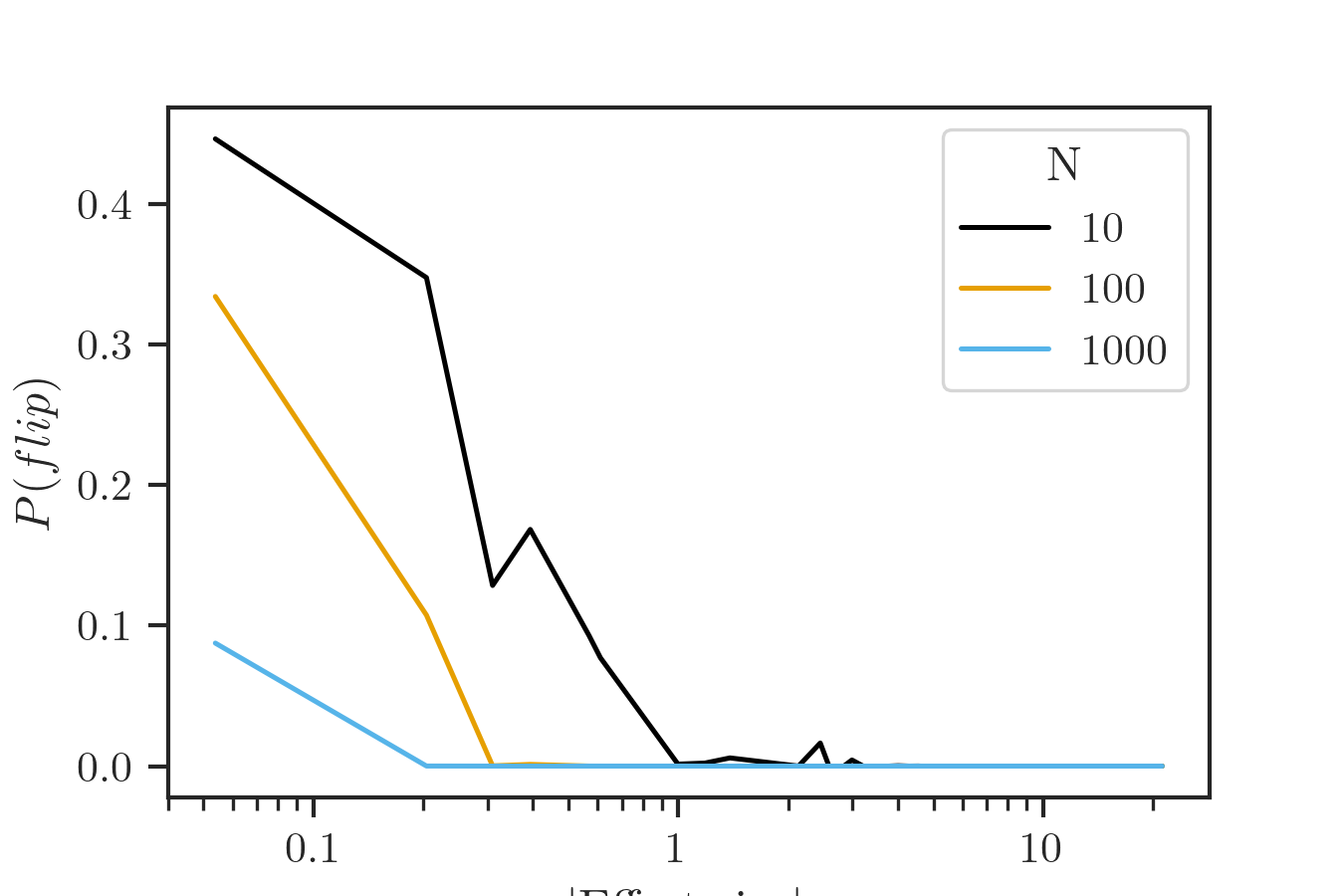}
  \caption{Sign-flip probability of a sampling-based estimator as a function of
  the exact absolute effect size, for per-condition budgets
  $N\in\{10,100,1000\}$ under default decoding. Reliability degrades sharply for
  small effects and low sampling budgets, while large effects are recovered
  reliably at any $N$.}
  \label{fig:p-flip}
\end{figure}

\paragraph{Findings.}
We observe three important findings. 
First, large and moderate effects are robust. 
Every parameter with $|E_{\text{exact}}|\ge 2$ (a cumulative $68.9\%$ of all effects)
has a flip probability below $0.01$ even at the smallest budget $N{=}10$. 
A researcher would recover the sign of these effects reliably. 
Second, the risk is concentrated entirely in the small-effect regime. 
Parameters with $|E_{\text{exact}}|<1$, 
which constitute $24.1\%$ of the factorial design, 
flip $18\%$ of the time at $N{=}10$ 
and still $6.3\%$ of the time at the commonly used budget $N{=}100$. 
In our study these small effects include several substantively interesting interactions 
(e.g., the near-null Qwen3$\times$China term), 
precisely the subtle country-of-origin signals a controlled study aims to detect. 
Third, the risk decays slowly: reducing the flip rate from $18\%$ to $1.3\%$ requires a $100\times$ increase in
sample size (from $N{=}10$ to $N{=}1000$), 
reflecting the $1/\sqrt{N}$ convergence of the estimator. 
Because the exact-PMF framework attains the $N\to\infty$ limit in
a single forward pass, it removes this reliability--compute tradeoff,
and guarantees that the sign and magnitude of every reported effect are free of
sampling artifacts.

\subsection{Extending the Framework to Prompt Framing}
\label{app:prompt-framing}

A key strength of casting behavioral evaluation as a fully crossed factorial
design is extensibility: any new source of variation can be introduced as an
additional factor and analyzed through the identical distributional-ANOVA
pipeline
(Theorem~\ref{thm:anova-consistency-main},
Corollaries~\ref{cor:main-effect-anova}--\ref{cor:paired-interaction-anova}).
We demonstrate this by promoting the administration prompt to a first-class
experimental factor, $\lambda_{\text{Prompt}}$. 
This yields a new capability. 
Rather than treating prompt phrasing as a fixed background choice, 
the framework isolates its causal effect on the measured construct, quantifies how it
interacts with each model, and recovers the latent trait with framing held under
statistical control.

We construct six system-prompt framings that preserve the measurement instrument exactly, namely, all retain the 1--7 scale, the agreement polarity, and the digit-only response schema, and hold every CETSCALE item verbatim. 
Only the surrounding framing wording varies (Table~\ref{tab:prompt-variants}). 
We evaluate a fully crossed $\textsc{Model}\times\textsc{Target}\times\textsc{Prompt}$ design over three models spanning the consensus and ethnocentrism ranges observed in the main study (Gemma~3~27B, Aya~Expanse~32B, Ministral~14B), four target countries, and the six framings, for $3\times4\times6\times17=1{,}224$ exact item-level PMFs.

\begin{table*}[htbp]
  \caption{Prompt framing variants. All preserve the 1--7 scale, agreement
  polarity, and digit-only response schema; every CETSCALE item is held verbatim.}
  \label{tab:prompt-variants}
  \centering
  \begin{small}
  \begin{tabular}{@{}p{3.4cm}p{8.6cm}@{}}
    \toprule
    \textbf{Variant} & \textbf{Framing} \\
    \midrule
    v0\_original & Consumer market-research persona; explicit 7-point scale; digit only. \\
    v1\_participant & Survey-participant framing; agreement scale 1--7; single digit only. \\
    v2\_concise & Minimal instruction; rate 1--7; output only the number. \\
    v3\_instructional & Task-style instruction; scale 1 (disagree) to 7 (agree); digit 1--7 only. \\
    v4\_roleplay & Everyday-shopper roleplay; 7 = agree, 1 = disagree; digit only. \\
    v5\_original\_no\_whitespace & v0 wording with whitespace/formatting normalized. \\
    \bottomrule
  \end{tabular}
  \end{small}
\end{table*}

\paragraph{Isolated effects under the three-factor design.}
Table~\ref{tab:prompt-effects} reports the full decomposition: 
The grand-mixture baseline $\mu_\emptyset$, the isolated main effects for all three factors, and the
three families of two-way interactions
($\textsc{Model}\times\textsc{Target}$, $\textsc{Model}\times\textsc{Prompt}$,
$\textsc{Target}\times\textsc{Prompt}$). 
Each row is an exact effect distribution summarized by mean shift ($\mathbb{E}$), predictive dispersion (SD),
signal-to-noise ratio (SNR), and discrete probability of direction (dPD).

Three capabilities of the framework are visible directly in Table~\ref{tab:prompt-effects}. 
First, the country-of-origin signal is recovered with framing controlled. 
The $\textsc{Target}$ main effects reproduce the pattern of the main study (Canada $+4.7$, China $-6.4$, France $-0.4$, USA $+2.1$), and every $\textsc{Target}\times\textsc{Prompt}$ interaction is small (all
$|\mathbb{E}|<1.5$, SNR $<0.4$, dPD near chance). 

Second, prompt framing is quantified as its own causal effect. 
The $\textsc{Prompt}$ main effect spans roughly $13$ composite-score points ($+6.9$ for v5 to $-6.7$ for v2), giving a direct, exact measurement of how administration framing shifts the construct
level. 

Third, model-specific framing sensitivity becomes measurable. 
The $\textsc{Model}\times\textsc{Prompt}$ block shows Gemma~3~27B interacting weakly with framing (all $|\mathbb{E}|<3.9$, SNR $<0.5$), 
whereas Aya~32B and Ministral~14B interact strongly (Aya$\times$v2 SNR $2.2$; Ministral$\times$v0
SNR $1.6$). 
Consistent with these interactions, the prompt-controlled $\textsc{Model}$ main effects recover the trait ordering of the main study, with Aya~32B highest ($+9.1$), Ministral~14B intermediate ($+4.0$), and Gemma~3~27B lowest ($-13.0$). 
The reduced magnitudes relative to Table~\ref{tab:coefficient-table-marginal} reflect the three-model subset and its distinct grand mean ($\mu_\emptyset=62.3$).

\begin{table*}[htbp]
  \caption{Full three-factor effect decomposition (Part 1 of 2: Main effects and Model $\times$ Target). Rows are isolated effect distributions summarized by mean shift ($\mathbb{E}$), predictive dispersion (SD), signal-to-noise ratio (SNR), and discrete probability of direction (dPD).}
  \label{tab:prompt-effects}
  \centering
  \begin{small}
  \begin{sc}
  \begin{tabular}{lrrll}
  \toprule
  Parameter & $\mathbb{E}$ & $\mathrm{SD}$ & $\mathrm{SNR}$ & $\mathrm{dPD}$ \\
  \midrule
  $\mu_{\emptyset}$ & 62.34 & 13.49 &  &  \\
  \midrule
  Aya-32B        &   9.08 & 12.67 & 0.72 & 0.76 \\
  Gemma 3-27B    & -13.03 & 11.74 & 1.11 & 0.90 \\
  Ministral-14B  &   3.96 &  9.61 & 0.41 & 0.70 \\
  \midrule
  Canada &   4.74 & 5.24 & 0.90 & 0.93 \\
  China  &  -6.38 & 6.08 & 1.05 & 0.97 \\
  France &  -0.44 & 4.78 & 0.09 & 0.48 \\
  USA    &   2.08 & 5.43 & 0.38 & 0.44 \\
  \midrule
  v0\_original       &   6.74 & 10.16 & 0.66 & 0.64 \\
  v1\_participant    &  -0.17 &  7.22 & 0.02 & 0.58 \\
  v2\_concise        &  -6.68 &  9.73 & 0.69 & 0.82 \\
  v3\_instructional  &  -4.08 &  7.76 & 0.53 & 0.62 \\
  v4\_roleplay       &  -2.67 &  6.94 & 0.38 & 0.68 \\
  v5\_no\_whitespace &   6.86 &  9.32 & 0.74 & 0.74 \\
  \midrule
  Aya-32B $\times$ Canada       &  0.31 & 2.14 & 0.15 & 0.35 \\
  Aya-32B $\times$ China        &  0.36 & 4.01 & 0.09 & 0.51 \\
  Aya-32B $\times$ France       &  0.08 & 3.33 & 0.03 & 0.53 \\
  Aya-32B $\times$ USA          & -0.75 & 3.47 & 0.22 & 0.48 \\
  \midrule
  Gemma 3-27B $\times$ Canada   &  2.14 & 3.28 & 0.65 & 0.62 \\
  Gemma 3-27B $\times$ China    & -4.40 & 9.86 & 0.45 & 0.59 \\
  Gemma 3-27B $\times$ France   & -1.83 & 5.38 & 0.34 & 0.55 \\
  Gemma 3-27B $\times$ USA      &  4.09 & 3.47 & 1.18 & 0.88 \\
  \midrule
  Ministral-14B $\times$ Canada & -2.45 & 3.08 & 0.80 & 0.77 \\
  Ministral-14B $\times$ China  &  4.04 & 4.52 & 0.89 & 0.82 \\
  Ministral-14B $\times$ France &  1.74 & 2.48 & 0.70 & 0.74 \\
  Ministral-14B $\times$ USA    & -3.34 & 4.10 & 0.81 & 0.77 \\
  \bottomrule
  \end{tabular}
  \end{sc}
  \end{small}
\end{table*}

\begin{table*}[htbp]
  \caption{Full three-factor effect decomposition (Part 2 of 2: Model $\times$ Prompt and Target $\times$ Prompt interactions).}
  \label{tab:prompt-effects-p2}
  \centering
  \begin{small}
  \begin{sc}
  \begin{tabular}{lrrll}
  \toprule
  Parameter & $\mathbb{E}$ & $\mathrm{SD}$ & $\mathrm{SNR}$ & $\mathrm{dPD}$ \\
  \midrule
  Aya-32B $\times$ v0        &  9.20 & 14.27 & 0.64 & 0.66 \\
  Aya-32B $\times$ v1        & -3.14 &  7.48 & 0.42 & 0.66 \\
  Aya-32B $\times$ v2        & -9.64 &  4.36 & 2.21 & 1.00 \\
  Aya-32B $\times$ v3        & -5.12 &  6.00 & 0.85 & 0.84 \\
  Aya-32B $\times$ v4        &  0.13 &  8.15 & 0.02 & 0.51 \\
  Aya-32B $\times$ v5        &  8.57 & 12.80 & 0.67 & 0.66 \\
  \midrule
  Gemma 3-27B $\times$ v0    & -0.99 &  5.35 & 0.18 & 0.47 \\
  Gemma 3-27B $\times$ v1    & -1.48 &  3.10 & 0.48 & 0.64 \\
  Gemma 3-27B $\times$ v2    &  3.89 &  2.14 & 1.82 & 0.97 \\
  Gemma 3-27B $\times$ v3    &  0.90 &  2.80 & 0.32 & 0.67 \\
  Gemma 3-27B $\times$ v4    & -0.11 &  3.41 & 0.03 & 0.44 \\
  Gemma 3-27B $\times$ v5    & -2.20 &  4.84 & 0.46 & 0.56 \\
  \midrule
  Ministral-14B $\times$ v0  & -8.22 &  5.09 & 1.61 & 1.00 \\
  Ministral-14B $\times$ v1  &  4.62 &  3.39 & 1.36 & 0.99 \\
  Ministral-14B $\times$ v2  &  5.75 &  3.77 & 1.53 & 0.95 \\
  Ministral-14B $\times$ v3  &  4.23 &  3.80 & 1.11 & 0.96 \\
  Ministral-14B $\times$ v4  & -0.01 &  3.81 & 0.00 & 0.64 \\
  Ministral-14B $\times$ v5  & -6.37 &  4.41 & 1.44 & 0.96 \\
  \midrule
  Canada $\times$ v0 & -1.04 & 2.88 & 0.36 & 0.80 \\
  Canada $\times$ v1 &  0.26 & 2.58 & 0.10 & 0.59 \\
  Canada $\times$ v2 &  1.28 & 3.36 & 0.38 & 0.57 \\
  Canada $\times$ v3 &  0.01 & 2.66 & 0.00 & 0.42 \\
  Canada $\times$ v4 & -0.22 & 2.27 & 0.10 & 0.36 \\
  Canada $\times$ v5 & -0.29 & 3.17 & 0.09 & 0.54 \\
  \midrule
  China $\times$ v0  &  0.46 & 3.56 & 0.13 & 0.65 \\
  China $\times$ v1  &  0.75 & 3.30 & 0.23 & 0.37 \\
  China $\times$ v2  & -1.45 & 4.88 & 0.30 & 0.46 \\
  China $\times$ v3  & -0.35 & 3.52 & 0.10 & 0.58 \\
  China $\times$ v4  &  0.76 & 3.64 & 0.21 & 0.45 \\
  China $\times$ v5  & -0.17 & 4.16 & 0.04 & 0.56 \\
  \midrule
  France $\times$ v0 &  0.42 & 3.43 & 0.12 & 0.37 \\
  France $\times$ v1 & -0.93 & 3.37 & 0.27 & 0.62 \\
  France $\times$ v2 & -0.21 & 3.38 & 0.06 & 0.46 \\
  France $\times$ v3 &  1.19 & 4.08 & 0.29 & 0.59 \\
  France $\times$ v4 & -0.55 & 3.44 & 0.16 & 0.41 \\
  France $\times$ v5 &  0.07 & 3.73 & 0.02 & 0.42 \\
  \midrule
  USA $\times$ v0    &  0.16 & 2.39 & 0.06 & 0.45 \\
  USA $\times$ v1    & -0.09 & 2.17 & 0.04 & 0.44 \\
  USA $\times$ v2    &  0.38 & 3.14 & 0.12 & 0.46 \\
  USA $\times$ v3    & -0.84 & 2.43 & 0.35 & 0.57 \\
  USA $\times$ v4    &  0.01 & 2.20 & 0.00 & 0.46 \\
  USA $\times$ v5    &  0.39 & 2.91 & 0.13 & 0.48 \\
  \bottomrule
  \end{tabular}
  \end{sc}
  \end{small}
\end{table*}

\end{document}